\pdfoutput=1
\documentclass[12pt,oneside]{article}

\usepackage[margin=1in]{geometry}

\usepackage{blindtext}
\usepackage[section]{placeins}
\usepackage{amssymb,amsmath,amsthm}
\usepackage{cancel}
\usepackage{natbib, url}
\usepackage{graphicx}
\usepackage{algorithm}
\usepackage{algpseudocode}
\usepackage{longtable}
\usepackage{hyperref}
\usepackage{multirow}
\usepackage{subcaption}
\usepackage{chngpage}
\usepackage{float}
\usepackage[svgnames]{xcolor}
\usepackage{listings}

\newtheorem{theorem}{Theorem}[section]
 \newtheorem{lemma}[theorem]{Lemma}

\newcommand{\R}{\mathbb{R}}

\newcommand{\bX}{{\boldsymbol X}}
\newcommand{\bY}{{\boldsymbol Y}}

\newcommand{\bH}{{\boldsymbol H}}
\newcommand{\bA}{{\boldsymbol A}}
\newcommand{\bB}{{\boldsymbol B}}
\newcommand{\bC}{{\boldsymbol C}}

\newcommand{\bE}{{\boldsymbol E}}

\newcommand{\bI}{{\boldsymbol I}}
\newcommand{\bPhi}{{\boldsymbol \Phi}}
\newcommand{\bDelta}{{\boldsymbol \Delta}}
\newcommand{\bzero}{{\boldsymbol 0}}
 \DeclareMathAlphabet{\pazocal}{OMS}{zplm}{m}{n}

\usepackage{lastpage}

\title{A Bayesian Vertical Federated Learning Framework for Multivariate Reduced-Rank High-Dimensional Regression}
\author{Brigham Halverson \\
  Department of Statistics \\
  Texas A\&M University \\
  \texttt{birghalvy@tamu.edu}
  \and
  Sharmistha Guha  \\
       Department of Statistics\\
       Texas A\&M University\\
       \texttt{sharmistha@tamu.edu}
       \and
       Jessica Bernard \\
       Department of Psychological \& Brain Sciences\\
       Texas A\&M University\\
       \texttt{jessica.bernard@tamu.edu}
        \and
       Rajarshi Guhaniyogi \\
       Department of Statistics\\
       Texas A\&M University\\
       \texttt{rajguhaniyogi@tamu.edu}}
\date{}

\begin{document}

\maketitle

\begin{abstract}
Federated learning (FL) has emerged as a leading privacy-preserving framework for collaborative machine learning across decentralized environments. While considerable progress has been made in horizontal federated learning (HFL), where data with common features is distributed across sites, vertical federated learning (VFL), where sites share observations across distinct feature sets, remains less explored. Advancing Bayesian high-dimensional multivariate reduced-rank regression methods for VFL poses unique challenges: (a) stringent privacy regulations preventing local site data from being shared, and (b) fitting local regressions overlooks essential modeling aspects such as inter-variable correlations. In contrast HFL allows each site to fit a comparable model independently. In this article, we present \emph{a novel Bayesian VFL framework for multivariate high-dimensional reduced-rank regression}, coined as \emph{BayesVFLReg}, which enables precise coefficient estimation while safeguarding both features and responses privacy.
Participating sites employ a shared random sketching matrix to compress local variables into privacy-preserving sketches. A central server collects these sketches where Bayesian multivariate reduced-rank regression is done using Gaussian scale mixture priors. For feature selection, we introduce a single-step post-processing strategy based on mixture-model clustering of the absolute posterior coefficient means to distinguish signal from noise per response variable. \emph{BayesVFLReg} is computationally scalable for large, high-dimensional datasets and facilitates efficient variable selection. Theoretically, we establish sharp non-asymptotic bounds on the posterior probability that the fitted density falls within a Hellinger ball centered at the true data-generating density. Comparative simulation studies and real-world data analyses show that \emph{BayesVFLReg} reliably identifies sparse feature effects, even under feature correlation.
 \end{abstract}


\section{Introduction}


Federated learning (FL) has sparked a paradigm shift in privacy-preserving and communication-efficient analytics, empowering collaborative model training across multiple institutions without direct data sharing. This innovation allows organizations to leverage distributed data resources while maintaining strict privacy standards, making FL a cornerstone in sensitive domains such as healthcare, finance, and multi-center research collaborations. FL is generally categorized into horizontal federated learning (HFL) and vertical federated learning (VFL). In HFL, participating sites share identical sets of features but differ in their observations, facilitating conventional statistical aggregation strategies. In contrast, VFL deals with sites that possess distinct feature sets for aligned subject observations. Within VFL, there exist multiple sub-classifications; following \citep{wu2025vertical}. In this article, we focus on the \emph{Precise VFL} setting, in which all observations are perfectly aligned across sites.

We consider inference on the parameters of a high-dimensional multivariate regression model, which can be described as:
\begin{equation}
   \bY_i^\top = \bX_i^\top\bC + \boldsymbol{e}_i^\top, \quad i \in \{1, \dots, n\},
   \label{eq:basic_reg}
\end{equation}
where $\bX_i$ is a $P \times 1$ vector of features for the $i$th subject, and $\bY_i$ is a $Q \times 1$ vector representing the response variables for the $i$th subject. The regression coefficient matrix $\bC$ is of size $P \times Q$ and contains the parameters of interest. The noise vector $\boldsymbol{e}_i$ accounts for residual variation, and is assumed to follow an independent $Q$-variate normal distribution with mean zero and a diagonal covariance structure, i.e., $\boldsymbol{e}_i \mid \{\sigma_1^2, ..., \sigma_Q^2\} \overset{iid}{\sim} N_Q(\boldsymbol{0}, \text{diag}(\sigma_1^2, ..., \sigma_Q^2))$.

By aggregating over all $n$ subjects, the model can be rewritten in matrix form as:
\begin{equation}
   \bY = \bX\bC + \boldsymbol{E},
   \label{eq:basic_reg_matrix}
\end{equation}
where $\bY$ is a $n \times Q$ matrix with the $i$th row given by $\bY_i^\top$, and $\bX$ is a $n \times P$ matrix of features with the $i$th row corresponding to $\bX_i^\top$. Similarly, $\boldsymbol{E}$ is the $n \times Q$ residual matrix with its $i$th row denoted by $\boldsymbol{e}_i^\top$.

A considerable amount of research has addressed the estimation of such regression models in the context of \emph{horizontal federated learning} (HFL) \citep{yang2019federated}. In HFL, data are distributed horizontally across multiple sites, meaning each site has a different subset of independent subjects but each subject is observed across the same set of $P$ features and $Q$ responses. This implies that for the model specified above, observations across sites remain independent. Consequently, the model can be fitted locally at each site using its site-specific data block, leading to local estimates of the parameters $\{\bC, \sigma_1^2, \ldots, \sigma_Q^2\}$. These local estimates, though potentially noisy due to limited sample sizes per site, can be subsequently transmitted to a central server, where they are aggregated to obtain final consensus estimates of the regression parameters. This federated approach preserves both data privacy (by not sharing raw data across sites) and allows for scalable analytical procedures across decentralized and heterogeneous datasets.

In contrast, vertical federated learning (VFL) is still a relatively underexplored field, especially when compared to horizontal federated learning. This is largely due to the intrinsic challenges of VFL: since features are distributed across different sites, no individual site can access the entire set of features for any given data sample.
This decentralized structure complicates the process of fitting a single, unified model, as information about each sample is distributed in a fragmented manner. To date, the majority of research on vertical federated learning (VFL) has emerged primarily within the machine learning community, with a predominant focus on developing algorithmic solutions for model training and prediction subject to privacy constraints \citep{liu2024vertical}. In the statistical literature, investigations under the VFL framework have been largely limited to simple linear and logistic regression models, with primary attention given to point estimation and hypothesis testing from a frequentist standpoint.
Notable studies have proposed methods for estimating model parameters and conducting statistical inference in these settings, often leveraging techniques such as secure computation, cryptographic protocols, and homomorphic encryption to safeguard data privacy \citep{gratton2018distributed, gascon2016secure, sanil2004privacy, chen2021homomorphic, zhao2023vflr}. Several studies have explored the application of high-dimensional ridge regression models in the context of VFL. 
\cite{yang2019federated} introduced a VFL algorithm tailored for ridge regression; however, its application is restricted to settings involving only two sites, and its reliance on gradient-based optimization methods may present efficiency challenges.
Subsequently, \cite{liu2021federated} incorporated Yang’s VFL algorithm into a FL framework designed for smart grids. Recent studies \citep{dai2022edge, cai2022efficient} have emphasized that large-scale datasets and the computational overhead of homomorphic encryption can considerably slow down the overall process for VFL participants. In response, these works propose VFL algorithms for ridge regression that require only a single communication round to achieve the optimal solution, thereby eliminating the need for iterative procedures. Although these approaches address point estimation, there is a notable lack of methods to \emph{draw inference with uncertainty quantification} for regression parameters within VFL. In fact, approximating posterior distributions for regression parameters in the VFL setting remains highly challenging due to the induced dependencies across feature effects located at different sites. The joint likelihood function does not readily decompose across sites as in HFL, thus hindering the use of standard distributed Bayesian computation strategies and making scalable posterior inference difficult \citep{hassan2024scalable}. 

\subsection{Our Proposal and Contributions}

To address the limitations inherent in joint parametric inference and privacy in vertical federated learning, we propose utilizing \textit{data sketching} \citep{ahfock2017statistical, dobriban2018new, guhaniyogi2025bayesian, guhaniyogi2025sketching} as a foundational aspect of our framework. In this approach, each site multiplies its feature matrices by a random compression matrix, which is predetermined and shared among all participating sites before analysis begins. This process yields a low-dimensional, sketched version of the original data at each site. Once the sketched features are computed locally, only these sketched representations, rather than the raw features, are sent to a central server. At the central server, a multivariate high-dimensional regression model, as described in equation (\ref{eq:basic_reg_matrix}), is then fitted using the pooled sketched data from all sites. 
Our privacy guarantee in VFL is fundamentally supported by an information-theoretic result: the random sketching transformation ensures that, as the sample size increases, the mutual information between the original features and the sketched data per feature per sample converges to zero \citep{zhou2009compressed}.
This implies that, even with a moderately large sample size, the original feature data at any site cannot be reconstructed or reliably inferred from the shared sketched features, which is a key aspect of our approach. Consequently, this property provides a robust foundation for information-theoretic privacy preservation during the collaborative learning process.

Our methodology employs a reduced-rank formulation for the coefficient matrix $\bC$, which enables the joint modeling of multiple correlated response variables in high-dimensional contexts while maintaining computational efficiency. To achieve adaptive shrinkage and sparsity, we impose a global-local scale mixture of Gaussian priors on the elements of $\bC$. 
This prior naturally shrinks the coefficients of irrelevant features toward zero and promotes sparsity across both the rows and columns of the coefficient matrix, thereby improving interpretability and predictive performance.
 Importantly, our approach is fully Bayesian, allowing for posterior inference not only on the regression coefficients $\bC$, but also on the error variances $\sigma_1^2, \ldots, \sigma_Q^2$. This capability provides \emph{coherent uncertainty quantification} for all model parameters. 
 We refer to the proposed approach as \emph{BayesVFLReg}, a new approach developed to fill the current gap in \emph{Bayesian vertical federated learning} methods for high-dimensional multivariate regression.
Furthermore, unlike existing VFL methods for linear regression, which often assume a single ``active site" that houses all response variables for the model, our framework supports a more flexible arrangement wherein response variables can be distributed across multiple active sites, thereby enabling broader and more versatile federated analyses in real-world multi-institutional settings. 

A central, defining strength of our proposed VFL approach lies in its rigorous theoretical foundation. Under justifiable assumptions, we prove that the posterior probability of the fitted density concentrates around the Hellinger ball of the true density. This occurs at a rate that closely mirrors the minimax optimal rate, seamlessly substituting traditional sample size with the dimensions of the data sketches. Our framework provides definitive, mathematical insights into how key model dynamics, specifically the number of predictor coefficients, sample size, and true regression model sparsity, directly dictate the sketching dimensions required to achieve optimal posterior concentration. The framework also identifies the exact, sufficient conditions required for both the sketching matrices and the true regression coefficients to consistently ensure desirable and highly reliable performance. Crucially, the literature on theoretically guaranteed uncertainty for Bayesian VFL is notably scarce, making the rigorous proofs developed in this work an exceptionally vital and timely contribution to the field.

Some clarifications are warranted. First, unlike differential privacy \citep{dwork2014algorithmic, guha2025differentially}, which achieves probabilistic privacy guarantees by injecting noise to obscure the influence of individual records, information theoretic privacy provides unconditional protection, ensuring that data sketches disclose minimal information about individuals, regardless of the capabilities of an adversary.
Second, this work advances the literature on data sketching. Earlier research has investigated sketching methods primarily in the context of simple linear regression, emphasizing accurate coefficient estimation, and has extended these approaches to ridge and lasso regression models, frequently analyzing their statistical properties through leverage scores \citep{sarlos2006improved, dobriban2018new}.
However, research on leveraging data sketching in high-dimensional Bayesian regression with large sample sizes, where computational costs rapidly escalate and data sketching could provide an effective solution, remains scarce. Exceptions are some recent studies that focus exclusively on univariate high-dimensional regression \citep{geppert2017random, guhaniyogi2025bayesian}.
To the best of our knowledge, our approach is the first to develop an efficient Bayesian computational algorithm for large-scale multivariate reduced-rank regression utilizing data sketching, enabling accurate and uncertainty-aware inference for regression parameters in this context.

The remainder of this paper is organized as follows. In Section~\ref{sec:VFL_MVR}, we introduce the proposed vertical federated learning framework for high-dimensional multivariate regression under a principled Bayesian approach. Section~\ref{sec:Simulation} demonstrates the effectiveness of data sketching through empirical studies, highlighting how our framework ensures accurate inference of regression parameters under vertical federated setting. In Section~\ref{sec:data_app}, we apply our method to real data to illustrate its practical utility. Section~\ref{sec:VFLreg_theory} develops the theoretical framework including assumptions and the main posterior concentration results. Finally, Section~\ref{sec:conclusion} concludes the paper with a discussion of our findings and potential avenues for future research. Appendix offers proofs of the theoretical results presented in Section~\ref{sec:VFLreg_theory}.

\section{Vertical Federated Learning for Multivariate Regression}
\label{sec:VFL_MVR}

Let $\mathcal{N} = \{1, \ldots, n\}$ represent the set of $n$ sample identifiers. Suppose there are $K$ distinct participating sites, each of which securely holds unique subsets of response and feature data corresponding to all sample IDs in $\mathcal{N}$. 
In this context, suppose there are $Q$ response variables and $P$ feature variables in total. The $k$th site holds $Q_k$ response variables and $P_k$ feature variables, for $k = 1, \ldots, K$, where $\sum_{k=1}^{K} Q_k = Q$ and $\sum_{k=1}^{K} P_k = P$. 

We allow flexibility in the storage at each site: a site may contain only responses ($P_k = 0$), only features ($Q_k = 0$), or both, provided that at least one variable type is present at the site, i.e., $\max(P_k, Q_k) > 0$. This design permits realistic configurations, such as sites specializing in either response or feature data.
Let $\bY^{(k)}\in \mathbb{R}^{n\times Q_k}$ denote the matrix of the $Q_k$ response variables for all sample IDs at the $k$th site, and let $\bX^{(k)}\in \mathbb{R}^{n\times P_k}$ denote the matrix of $P_k$ feature variables for all sample IDs at the $k$th site. Importantly, the raw data at each site is \emph{privacy-protected}. This implies that the data must remain local and cannot be exchanged with other sites or transmitted to a central server.

In the context of multivariate reduced-rank regression, our objective is to incorporate all response and feature variables across sites to estimate regression coefficients and identify key features corresponding to each response. However, independent analysis at each site is not feasible due to the inherent limitations: many sites may exclusively contain either response or feature data, but not both. Furthermore, performing regression analysis using only a subset of features available at a given site neglects the inter-feature correlations and can substantially bias the regression coefficients. This situation exemplifies the challenges of vertical federated learning (VFL) paradigm, which is notably more challenging than horizontal federated learning. In horizontal federated learning, each site possesses data for a distinct subset of samples, but includes all features and responses for those samples. As a result, the full model can be distributed across sites, and local inferences made at each site are relatively straightforward, typically offering noisy approximations to the global posterior distribution of regression parameters. Existing approaches leverage techniques such as secure computation, cryptographic protocols, and homomorphic encryption to protect data privacy \citep{gratton2018distributed, gascon2016secure, sanil2004privacy, chen2021homomorphic, zhao2023vflr} in VFL, but do not account for parametric uncertainty in high-dimensional linear regression. In what follows, we develop a two-stage VFL framework to address these challenges, with the overall procedure depicted in Figure~\ref{fig:model_diagram}.
\begin{figure}[ht]
    \centering
    \includegraphics[width=.75\linewidth]{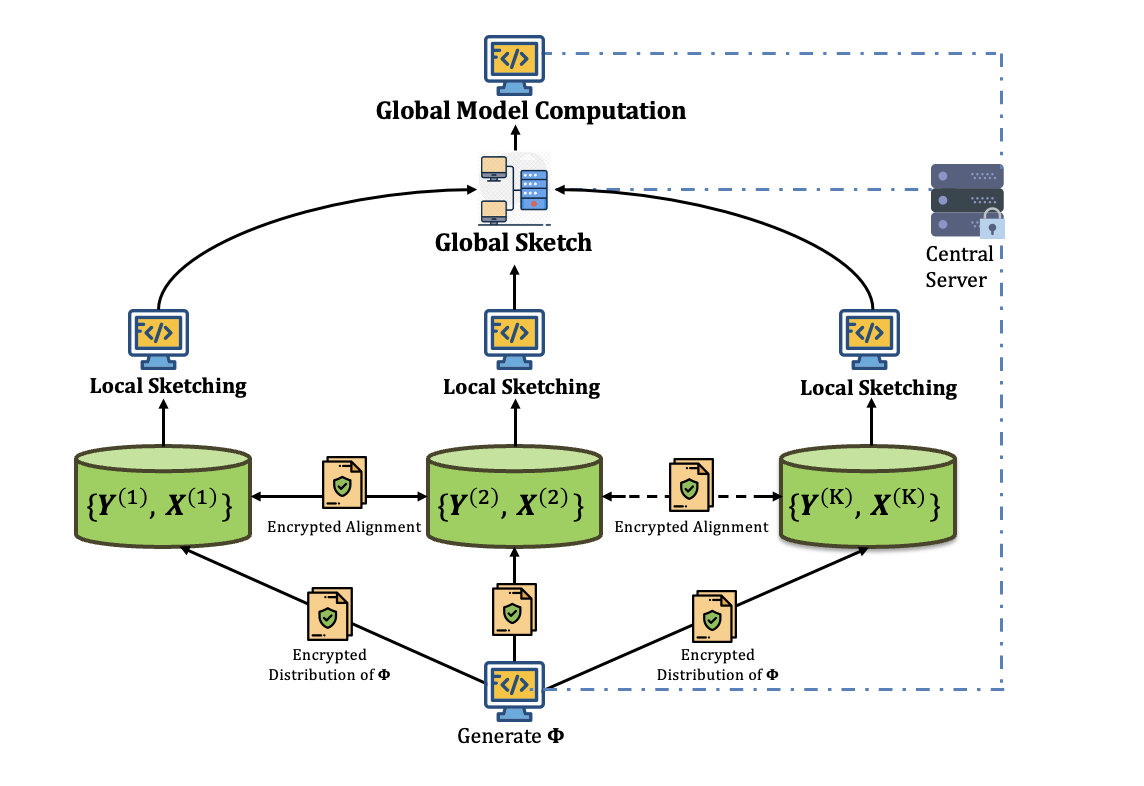}
    \caption{VFL for Multivariate Regression using Federated Sketch Construction}
    \label{fig:model_diagram}
\end{figure}
\subsection{Step 1: Federated Sketch Construction with Random Matrices in Each Site}
We propose constructing a random matrix $\boldsymbol{\Phi}$ of dimension $m \times n$, with each entry drawn i.i.d. from $N(0, 1/n)$ and $m << n$. This random matrix is shared across all $K$ sites. Each site $k$ performs a \emph{local and private computation}, creating a small, $m$-dimensional sketch of its own data:
\begin{itemize}
    \item Local Sketched Responses: $\bY_{\boldsymbol{\Phi}}^{(k)} = \boldsymbol{\Phi} \bY^{(k)}$
    \item Local Sketched Features: $\bX_{\boldsymbol{\Phi}}^{(k)} = \boldsymbol{\Phi} \bX^{(k)}$
\end{itemize}
These \emph{local sketches} are then securely transmitted to a central server. The \emph{global sketch} is reconstructed via a simple, communication-efficient summation:
\begin{itemize}
 \item Global Sketched Responses: $\bY_{\boldsymbol{\Phi}} = [\bY_{\boldsymbol{\Phi}}^{(1)}:\cdots:\bY_{\boldsymbol{\Phi}}^{(K)}]\in \mathbb{R}^{m\times Q}$
    \item Global Sketched Features: $\bX_{\boldsymbol{\Phi}} = [\bX_{\boldsymbol{\Phi}}^{(1)}:\cdots: \bX_{\boldsymbol{\Phi}}^{(K)}]\in \mathbb{R}^{m\times P}$  
\end{itemize}
A fundamental strength of our federated sketching framework lies in its inherent privacy protection. The random projection mechanism at its core ensures that the original data remain secure, providing rigorous information-theoretic privacy guarantees. Specifically, when a central server receives a local sketch, such as $\bX_{\boldsymbol{\Phi}}^{(k)} = \boldsymbol{\Phi} \bX^{(k)}$ and $\bY_{\boldsymbol{\Phi}}^{(k)} = \boldsymbol{\Phi} \bY^{(k)}$, it encounters a highly underdetermined linear system. Because the sketch dimension $m$ is much smaller than the original data dimension $n$, it is not possible to uniquely reconstruct $\bX^{(k)}$ or $\bY^{(k)}$ from their respective sketches.

This privacy property is mathematically formalized via information theory. The mutual information between the sketch and the original data, denoted $I(\bX_{\boldsymbol{\Phi}}^{(k)}, \bX^{(k)})$, quantifies the amount of information the sketch reveals. It can be established that the supremum of the normalized mutual information, $\frac{I(\bX_{\boldsymbol{\Phi}}^{(k)}, \bX^{(k)})}{nP_k}$, is bounded above by $O\left(\frac{m}{n}\right)$ \citep{zhou2008compressed}. An analogous result holds for the sketch of the responses: $\frac{I(\bY_{\boldsymbol{\Phi}}^{(k)}, \bY^{(k)})}{nQ_k}$ is also bounded above by $O\left(\frac{m}{n}\right)$. Notably, this guarantee persists even in the adversarial scenario where the sketching matrix $\boldsymbol{\Phi}$ is fully known to the server. 
As the dataset size  $n$ grows, these upper bounds approach zero, demonstrating that in the limit as $n \rightarrow \infty$, the information contained in the sketch about the original dataset becomes negligible. In effect, the sketched data becomes as anonymous as a random sample independent of the original data.

In real-world federated settings, the privacy advantages are often even more pronounced. Conventionally, only the global sketch (such as $\bX_{\boldsymbol{\Phi}}^{(k)}$) is shared, and the random sketching matrix $\boldsymbol{\Phi}$ may remain confidential. This added layer of protection dramatically increases the difficulty of reconstructing the original data, surpassing even the theoretical guarantees \citep{guhaniyogi2025bayesian, guhaniyogi2025sketching}. It is important to note that this form of privacy, rooted in the principles of information theory and computational hardness, is fundamentally distinct from statistical privacy guarantees such as those provided by $\epsilon$-differential privacy \citep{dwork2014algorithmic}. Whereas $\epsilon$-differential privacy constrains the probability of distinguishing individual samples, our framework ensures that the sketches intrinsically reveal minimal information about the original data, irrespective of computational resources or adversarial knowledge.

\subsection{Step 2: Fitting Reduced-Rank Multivariate Regression in the Central Server}
At the central server, we fit the multivariate regression model with the global sketches, given by,
\begin{equation}
    \bY_{{\boldsymbol \Phi}} = \bX_{{\boldsymbol \Phi}}\bC + \boldsymbol{E},
    \label{eq:mds_model}
\end{equation}
where $\boldsymbol{E} = [\boldsymbol{e}_1^\top: \hdots: \boldsymbol{e}_m^\top]^\top$ denotes the $m\times Q$ dimensional matrix of errors, with $\boldsymbol{e}_1,...,\boldsymbol{e}_m\stackrel{i.i.d.}{\sim} N_Q({\boldsymbol 0}, diag(\sigma_1^2,...,\sigma_Q^2))$.

Model~\ref{eq:mds_model} closely resembles a ${\boldsymbol \Phi}$-transformed version of the original statistical model, but with a fundamental distinction. Specifically, we do \emph{not} apply the ${\boldsymbol \Phi}$ transformation to the error matrix. This design choice is rooted in insights from random matrix theory. In particular, it has been shown that the deviation $\left\| \boldsymbol{\Phi} \boldsymbol{\Phi}^\top - \bI \right\|$ can be bounded by $\tilde{C} \sqrt{\frac{m}{n}}$ with high probability, where $\tilde{C}$ is a constant and $m \ll n$ \citep{vershynin2010introduction}. Consequently, the ${\boldsymbol \Phi}$-transformed errors exhibit behavior that is approximately orthogonal, preserving essential statistical properties such as independence and Gaussianity in the sketches. By leaving the error matrix untransformed, we avoid sharing the $\boldsymbol{\Phi}$ matrix with the central server, as well as gaining in computation when the sample size $n$ is large.

We assume that the matrix $\bC$ possesses a low-rank structure, specifically $\text{rank}(\bC) = R < \min(P, Q)$. This assumption enables us to express $\bC$ as the product of two matrices: $\bC = \boldsymbol{B}\boldsymbol{A}^\top$, with $\boldsymbol{B} \in \mathbb{R}^{P \times R}$ and $\boldsymbol{A} \in \mathbb{R}^{Q \times R}$. This factorization is referred to as the reduced-rank representation and is commonly employed in multivariate statistical modeling \citep{reinsel1998multivariate, chen2012sparse} to capture latent structure and to prevent overfitting in high-dimensional regression settings.

Driven by practical applications in multivariate regression, especially those involving a large number of features and a comparatively moderate number of responses, we focus on the frequent scenario where $Q < P$. Although it is possible to place priors directly on the rank $R$ to facilitate adaptive inference on the latent space, we instead pursue a more flexible approach: we set $R = Q$, such that both $\boldsymbol{A}$ and $\boldsymbol{B}$ are of full column rank. To mitigate overfitting and limit the impact of irrelevant columns, we place continuous shrinkage priors on the columns of $\boldsymbol{B}$. 
This approach is analogous to widely used Bayesian multi-way regression methods, where explicit rank selection in low-rank matrix or tensor decompositions is frequently substituted with the choice of a relatively large rank combined with shrinkage estimation 
\citep{guha2021bayesian, guhaniyogi2021bayesian}. 

More specifically, we adopt a global-local continuous shrinkage prior framework (e.g., \cite{carvalho2010horseshoe, caron2008sparse}), realized as global-local scale mixtures of Gaussians \citep{polson2010shrink}. Under this arrangement, the $(p,q)$th entry $b_{p,q}$ of the matrix $\boldsymbol{B}$ is modeled as follows:
\begin{equation}
    b_{p,q} \mid \lambda_{p,q}, \tau_q \sim \mathcal{N}\left(0, \lambda_{p,q}^2\tau_q^2\right), \qquad \lambda_{p,q} \sim f_1, \qquad \tau_q \sim f_2,
\end{equation}
where $\tau_q$ is a global shrinkage parameter for the $q$th column, and $\lambda_{p,q}$ is a local shrinkage parameter for the $p$th entry of the $q$th column. The densities $f_1$ and $f_2$ are supported on $\mathbb{R}$, providing flexibility in prior specification. This global-local shrinkage setup parallels approaches commonly used in high-dimensional sparse regression, with the shrinkage inducing effective sparsity and controlling model complexity.

Different choices for $f_1$ and $f_2$ result in distinct forms of shrinkage priors. While our framework can accommodate any global-local normal mixture prior, for demonstration and empirical evaluation, we focus on the horseshoe prior \citep{carvalho2010horseshoe}, which is obtained by choosing both $f_1$ and $f_2$ as independent Half-Cauchy distributions. The horseshoe prior is well-known for its ability to strongly shrink irrelevant coefficients while retaining important signals. Each entry of $\boldsymbol{A}$ is assigned a N(0,1) prior distribution. Furthermore, for the variance parameters $\sigma^2_q$, we assign independent Jeffreys priors, specified as $\pi(\sigma^2_1) \propto 1/\sigma^2_q$. This noninformative prior supports robust estimation of variance components within our model.

\subsubsection{Posterior Approximation}
With prior distributions on $\boldsymbol{B}$ and $\boldsymbol{A}$ set as above, the posterior computation using a blocked Metropolis-within-Gibbs algorithm cycles through updating the full conditional distributions as expressed in Algorithm~\ref{alg:gibbs_sampling}.

\begin{algorithm}
\caption{Metropolis-within-Gibbs Algorithm}
\label{alg:gibbs_sampling}
\begin{algorithmic}
    \State Initialize $\boldsymbol{B}^{(0)}, \boldsymbol{A}^{(0)}, \{\sigma_q^{2,(0)}\}_{q = 1}^Q,\{\lambda_{p,q}^{(0)}\}_{p, q = 1}^{P,Q}, \{\tau_q^{(0)}\}_{q = 1}^Q$
    \For{$t = 1$ to $T_{samps}$}
        \State \textbf{Step 1: }Draw $\boldsymbol{B}^{(t)}$ from $\boldsymbol{B} \mid \bX_{{\boldsymbol \Phi}}, \bY_{{\boldsymbol \Phi}}, \boldsymbol{A}^{(t-1)}, \{\sigma_q^{2, (t - 1)}\}_{q = 1}^Q,\{\lambda_{p,q}^{(t-1)}\}_{p, q = 1}^{P,Q}, \{\tau_q^{(t-1)}\}_{q = 1}^Q$
    
        \State \textbf{Step 2: }Draw $\boldsymbol{A}^{(t)}$ from $\boldsymbol{A} \mid \bX_{{\boldsymbol \Phi}}, \bY_{{\boldsymbol \Phi}}, \boldsymbol{B}^{(t)}, \{\sigma_q^{2, (t-1)}\}_{q = 1}^Q$
    
        \State \textbf{Step 3: }Draw $\{\sigma_q^{2, (t)}\}_{q = 1}^Q$ from $\sigma_q^2 \mid \bX_{{\boldsymbol \Phi}}, \bY_{{\boldsymbol \Phi}}, \boldsymbol{A}^{(t)}, \boldsymbol{B}{(t)}$
    
        \State \textbf{Step 4: }Draw $\{\lambda_{p,q}^{(t)}\}_{p, q = 1}^{P,Q}$ from $\lambda_{p,q} \mid \tau_q^{(t-1)}, \boldsymbol{B}^{(t)}$ and $\{\tau_q^{(t)}\}_{q = 1}^Q$ from $\tau_q \mid \{\lambda_{p,q}^{(t)}\}_{p = 1}^P, \boldsymbol{B}^{(t)}$
    \EndFor
\end{algorithmic}
\end{algorithm}
Explicit expressions for steps 1-4 for the horseshoe shrinkage prior are available in the supplementary material. Although updating 2-4 is computationally straightforward, updating 1 faces computational challenges. To see this, let ${\boldsymbol \beta}=vec(\boldsymbol{B})^\top\in\mathbb{R}^{PQ\times 1}, \boldsymbol{\Lambda} = \text{diag}(\lambda_{11}^2\tau_1^2, ..., \lambda_{1Q}^2\tau_Q^2, ..., \lambda_{PQ}^2\tau_Q^2) \in \mathbb{R}_+^{PQ \times PQ}$, ${\boldsymbol y}_{\boldsymbol{\Phi}}=\text{vec}({\boldsymbol Y}_{{\boldsymbol \Phi}}^\top)\in\mathbb{R}^{mQ\times 1}$, $\boldsymbol{\Sigma} = \text{diag}(\sigma_1^2, ..., \sigma_q^2) \in \mathbb{R}_+^{Q \times Q}$ and $\tilde{\boldsymbol{\Sigma}} = \text{diag}(\boldsymbol{\Sigma}, ..., \boldsymbol{\Sigma}) \in \mathbb{R}_+^{mQ \times mQ}$. The full conditional distribution of ${\boldsymbol \beta}$ is given by,
\begin{equation}
    \boldsymbol{\beta} \mid \bX_{{\boldsymbol \Phi}}, \bY_{{\boldsymbol \Phi}}, \boldsymbol{A}, \{\sigma_q^2\}_{q = 1}^Q,\{\lambda_{p,q}\}_{p, q = 1}^{P,Q}, \{\tau_q\}_{q = 1}^Q \sim N_{PQ}(\boldsymbol{\Omega}_B^{-1}\tilde{\bX}^\top\tilde{\bY}, \boldsymbol{\Omega}_B^{-1})
\end{equation}
where $\tilde{\bY}_{\boldsymbol{\Phi}} = \tilde{\boldsymbol{\Sigma}}^{-1/2}\bY_{\boldsymbol{\Phi}}$, $\tilde{\bX}_{\boldsymbol{\Phi}} = \tilde{\boldsymbol{\Sigma}}^{-1/2}(\bX_{\boldsymbol{\Phi}} \otimes \boldsymbol{A})$, $\boldsymbol{\Omega}_B = (\tilde{\bX}_{\boldsymbol{\Phi}}^\top\tilde{\bX}_{\boldsymbol{\Phi}} + \boldsymbol{\Lambda}^{-1})$.

Given that in most applications $P >> m$, performing the first step can be computationally intensive, with complexity $O((PQ)^3)$. To address this, we take an approach adapted from \cite{guha2020bayesian}, and implement the following algorithm: 
\begin{algorithm}
\caption{Sampling $\boldsymbol{\beta}$}
\label{alg:beta_sample}
\begin{algorithmic}
    \State \textbf{Step 1: }Sample $\boldsymbol{u} \sim N(\boldsymbol{0}, \boldsymbol{\Lambda})$, ${\boldsymbol \delta} \sim N(0, I_{mQ})$
    
    \State \textbf{Step 2: }Set $\boldsymbol{v} = \tilde{\bX}_{\boldsymbol{\Phi}}\boldsymbol{u} + {\boldsymbol \delta}$
    
    \State \textbf{Step 3: }Solve $(\tilde{\bX}_{\boldsymbol{\Phi}}\boldsymbol{\Lambda}\tilde{\bX}_{\boldsymbol{\Phi}}^\top + I_{mq})\boldsymbol{w} = (\tilde{\bY}_{\boldsymbol{\Phi}} - \boldsymbol{v})$ to obtain $\boldsymbol{w}$
    
    \State \textbf{Step 4: }Set $\boldsymbol{\beta} = \boldsymbol{u} + \boldsymbol{\Lambda}\tilde{\bX}_{\boldsymbol{\Phi}}^\top\boldsymbol{w}$
\end{algorithmic}
\end{algorithm}

This reduces the complexity from $O(P^3Q^3)$ to $O(Q^3\text{min}(m^3, P))$, resulting in linear scaling with the number of features $P$ and cubic scaling with respect to $m$. Since $m \ll n$, this demonstrates that the method is highly scalable even when both $n$ and $P$ are large. 

\subsubsection{Post-Processing and Variable Selection}
\label{sec:post_process}

It is important to note that, although the shrinkage prior pulls unimportant coefficients toward zero, it does not set them exactly to zero. To identify the most relevant variables, we apply a post-processing procedure based on the method proposed in \cite{guha2021bayesian}, as described below. This approach uses Gaussian clustering to divide the predictors into two groups: those associated with zero and those with non-zero coefficients. The procedure is as follows:
\begin{algorithm}
\caption{Feature Classification Algorithm}
\label{alg:feature_classification}
\begin{algorithmic}
    \State \textbf{Step 1: }Obtain a posterior or mean estimate of $\bC$, call this $\widehat{\bC}$
    
    \State \textbf{Step 2: }Let $\tilde{\bC}$ be the absolute value applied to each element of $\widehat{\bC}$
    
    \State \textbf{Step 3: }Cluster the rows of $\tilde{\bC}$ into two groups using a two-component mixture of Gaussian distributions 
    
    \State \textbf{Step 4: }Use the generated probability ($p_j$, $j = 1, \dots, p$) that a given feature is allocated to the mixture component with the lowest mean
    
    \State \textbf{Step 5: }For $\tilde{\alpha} = 0.05$, use the above probabilities to classify if the given feature is non-zero, i.e., $p_j < \tilde{\alpha} \Rightarrow$ feature $j$ is classified as non-zero
\end{algorithmic}
\end{algorithm}

The parameter $\tilde{\alpha}$ serves as a model selection threshold. In this work, we set $\tilde{\alpha} = 0.05$ to ensure a low probability of misclassifying a variable as non-zero, that is, we require at least 95\% confidence that a selected variable truly differs from zero.

\section{Simulation}
\label{sec:Simulation}
This section evaluates the efficacy of our proposed \emph{BayesVFLReg} framework in estimating the coefficient matrix $\bC$ from model~(\ref{eq:mds_model}), leveraging induced reduced-rank priors. We present two different simulation scenarios. \textbf{Simulation 1} examines the performance of \emph{BayesVFLReg} on simulated data from~Model (\ref{eq:basic_reg_matrix}) with $n=500$, $P=300$, and $Q=5$. Here, we analyze our approach alongside Model~(\ref{eq:basic_reg}) applied to a centralized model that makes use of unsketched global data. This allows for the comparison of the posterior distributions of $\bC$ under varying privacy levels and sparsity degrees. These results provide insight into how the \emph{BayesVFLReg} approach stacks up against the ``gold standard" centralized setting, where all data are observed on a central server, revealing the impact of privacy constraints and model sparsity on estimation accuracy. Additionally, we assess computational efficiency of the \emph{BayesVFLReg} approach.

\textbf{Simulation 2} examines larger-scale data ($n=1{,}000$, $P=1{,}000$, $Q=5$), where direct benchmarking with full Bayesian computation of Model~(\ref{eq:basic_reg}) is computationally infeasible. In this context, we evaluate the \emph{BayesVFLReg} method against a state-of-the-art federated competitor, distributed ridge regression (\emph{dRidge}) \citep{gratton2018distributed}, as well as two efficient non-federated alternatives that operate on data centrally stored but offer greater computational tractability compared to full Bayesian modeling of Model~(\ref{eq:basic_reg}) for estimating $\bC$. Specifically, one approach follows a frequentist paradigm \citep{chen2012sparse}, while the other employs a variational Bayesian framework \citep{ruffieux2017efficient}. This comprehensive comparison enables us to rigorously assess both the accuracy and computational efficiency of \emph{BayesVFLReg}, relative to leading federated algorithms and conventional non-federated strategies for high-dimensional multivariate regression.

\subsection{Simulation 1: Effect of Privacy}

\textbf{Simulation 1} is designed to assess the performance of \emph{BayesVFLReg} across different levels of sparsity and privacy. We simulate $n = 500$ data observations from the $Q$-variate high-dimensional regression model (\ref{eq:basic_reg_matrix}), with $Q=5,$ $P=300$ and $\sigma_h = 1$ for all $h$. Each $P$-dimensional feature vector $\bX_i$ is simulated from a $N({\boldsymbol 0}, {\boldsymbol \Sigma})$, where 
   ${\boldsymbol \Sigma} = (1 - \rho)\bI_P + \rho\boldsymbol{J}_P,$ and $\boldsymbol{J}_P$ is a matrix of size $P \times P$ with all entries equal to 1. We fix $\rho = 0.5$ to ensure moderate correlation between all features.

The coefficient matrix ${\boldsymbol C}={\boldsymbol B}{\boldsymbol A}^\top$ is constructed by first simulating the entries of the matrix $\boldsymbol{A}$ from $N(0,1)$, and subsequently constructing $\boldsymbol{B}$ such that exactly $s$ rows of $\boldsymbol{B}$ are non-zero (with each non-zero entry sampled from N(0,1)). The simulation study varies $s$ across $10, 30$, and $50$ to represent different degrees of sparsity.
The constructed $\bC$, together with the simulated feature vectors $\bX_i$, is used to simulate the response matrix $\bY$. 

Our central objective is to study the influence of privacy and sparsity on the accuracy of the proposed \emph{BayesVFLReg} framework. Privacy protection is tuned by varying the sketch dimension, setting $m=50,100,250$. This results in normalized mutual information between $\bX$ and $\bX_\Phi$ (and between $\bY$ and $\bY_\Phi$) approximately $0.0005 (0.043)$, $0.001 (0.087)$, and $0.003 (0.217)$ bits, respectively, as calculated under the linear Gaussian model \citep{CoverThomas2006}. These values indicate a high degree of privacy in each simulation setting. In this context, bits quantify the adversary’s ability to infer the true data: each bit of leaked information halves the set of feasible guesses. Consequently, releasing a sketch that leaks $k$ bits of information reduces the adversary's plausible guesses by a factor of $2^k$ \citep{CoverThomas2006}.
We compare our approach to the ``gold standard" Bayesian multivariate reduced rank regression fit on the unsketched data in a central server, referred to as the \emph{centralized model}, repeating the analysis over 50 simulated datasets to ensure robustness.

We evaluate inferential performance using
\begin{align}
    \text{MSE} = \|\text{vec}(\widehat{\bC}) - \text{vec}(\bC)\|^2_2/qp, \text{      } \text{MSE}_{nz} = \|\text{vec}(\widehat{\bC}_{nz}) - \text{vec}(\bC_{nz})\|^2_2/qs,
\end{align} 
where $\widehat{\bC}$ is the posterior mean estimate of the entire coefficient matrix $\bC$, and $\widehat{\bC}_{nz}$ represents the posterior mean for the true non-zero entries (denoted by $\bC_{nz}$).

To further evaluate the accuracy of the estimated posterior distributions of the coefficient entries, we use the centralized model as the baseline and introduce a metric based on the Hellinger distance. This metric quantifies how closely the posterior for each coefficient obtained via the \emph{BayesVFLReg} approach, $\pi_m(c_{jh}\mid \bY_{\phi}, \bX_\phi)$, approximates the centralized posterior $\pi(c_{jh} \mid \bY, \bX)$:
\begin{equation}
    \text{Accuracy}_{j,h,m} = 1 - \frac{1}{2}\int_{\bC}\left(\sqrt{\pi_m(c_{jh}\mid \bY_{\phi}, \bX_\phi)} - \sqrt{\pi(c_{jh} \mid \bY, \bX)}\right)^2 \, dc_{jh}
\end{equation}
This accuracy metric, $\text{Accuracy}_{j,h,m}$, ranges from $0$ to $1$, with values near $1$ indicating strong agreement between the \emph{BayesVFLReg} and centralized posteriors. To provide more detailed insights, we summarize accuracy separately for zero and non-zero coefficients as follows:
\begin{align*}
\text{Accuracy}_{m,nz} = \frac{1}{sQ}\sum_{j,h: c_{jh} \neq 0} \text{Accuracy}_{j,h,m}, \\ 
\text{Accuracy}_{m,z} = \frac{1}{(P - s)Q}\sum_{j,h: c_{jh} = 0} \text{Accuracy}_{j,h,m}.
\end{align*}
Since the posterior distribution of $c_{jh}$ cannot be computed in closed form, we approximate both $\pi_m$ and $\pi$ numerically. Posterior samples of $\boldsymbol{B}$ and $\boldsymbol{A}$ are used to construct samples of $\bC$; subsequently, we employ the Bernstein-von Mises (Bayesian Central Limit Theorem) approximation as follows:
\begin{enumerate}
    \item Collect $B$ posterior draws of $c_{ij}^{(m)}$ and $c_{ij}$ from $c_{ij}^{(m)} \mid \bY_\phi, \bX_\phi$ and $c_{ij} \mid \bY, \bX$ respectively;
    \item Estimate the posterior mean and variances; $\Tilde{\mu}_m = \frac{1}{B}\sum_{b = 1}^B c_{ij}^{(m,b)}, \Tilde{\sigma}^2_m = \frac{1}{B - 1}\sum_{b = 1}^B(c_{ij}^{(m,b)} - \Tilde{\mu}_m)^2, \Tilde{\mu} = \frac{1}{B}\sum_{b = 1}^B c_{ij}^{(b)}, \Tilde{\sigma}^2 = \frac{1}{B - 1}\sum_{b = 1}^B(c_{ij}^{(b)} - \Tilde{\mu})^2$;
    \item Then given a random grid with values $c_{ij}^{(f)}$ for $f = 1, 2, \hdots, F$, 
    \begin{equation}
        \text{Accuracy}_{j,h,m} \approx 1 - \frac{1}{2F}\sum_{f = 1}^F\left(\sqrt{N(c_{ij}^{(f)} \mid \Tilde{\mu}_m, \Tilde{\sigma}^2_m)} - \sqrt{N(c_{ij}^{(f)} \mid \Tilde{\mu}, \Tilde{\sigma}^2)}\right)^2
        \label{eq:accuracy}
    \end{equation}
\end{enumerate}
This simulation also illustrates the computational efficiency gained by incorporating the global sketch. Let ESS$_m$ be the total effective sample size of $\bC$, obtained by summing the effective sample sizes across all coefficients when using $\boldsymbol{\Phi}$ with rank$(\boldsymbol{\Phi}) = m$ to generate 5,000 posterior samples (from 5 independent chains of 1,000 draws each with a burn-in of 200) in $T_m$ seconds.
\begin{equation}
    \text{Computational Efficiency}_m = \log_2\text{ESS}_m/T_m,
    \label{eq:comp_eff}
\end{equation}
\noindent where ESS$_m$ across features is computed using the \verb|coda| package in \verb|R|. Computational Efficiency is also computed for the non-federated posterior. All metrics are reported as averages across the 50 replicated datasets in each scenario.

\subsubsection{Simulation 1 Results}

We fit the models using MCMC chains with 5,000 draws, discarding the first 2,000 samples as burn-in and retaining the remaining 3,000 draws for posterior analysis.
Figure \ref{fig:MSE_overall} depicts the MSE and MSE$_{nz}$, respectively. Figure \ref{fig:MSE_overall} shows that, as privacy level decreases (with the increase of sketching dimension $m$), both MSE and $\text{MSE}_{nz}$ improve, indicating higher inferential accuracy. Furthermore, performance depends critically on the ratio $s/m$: as this ratio becomes smaller, the \emph{BayesVFLReg} model’s accuracy approaches that of the model fit to centralized data. This finding highlights that higher levels of sparsity allow for stronger privacy protection while maintaining nearly similar inference as that from the centralized model.

\begin{figure}[ht]
  \centering
    \includegraphics[width=1\linewidth]{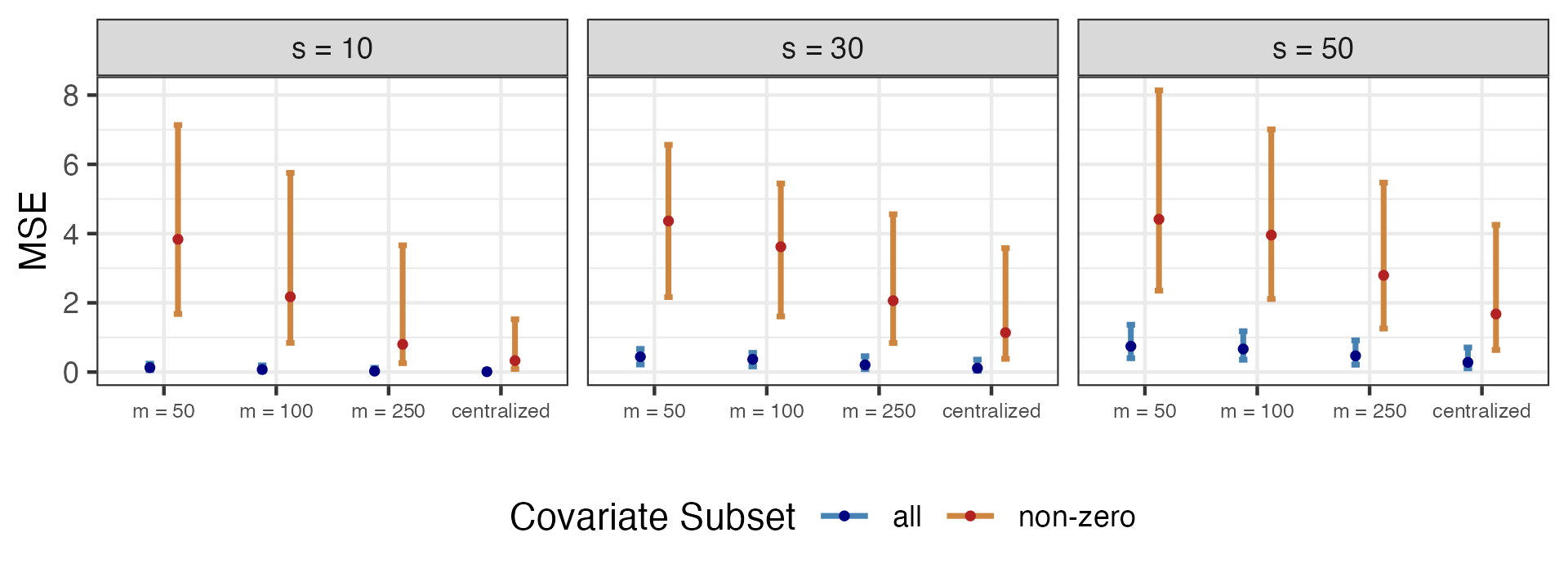}
  
  \caption{Figure shows a summary of the MSE (blue) and MSE$_{nz}$ (red) across the 50 replicates. The dots denote mean MSE with intervals around the mean showing standard errors over 50 replications. As privacy decreases (larger $m$), accuracy improves. At higher-degree of sparsity (smaller $s$), more privacy can be guaranteed while achieving competitive performance relative to the uncompressed model.}
  \label{fig:MSE_overall}
\end{figure}

Table \ref{tab:sim1_results} depicts both the accuracy and the computational efficiency at different privacy levels ($m = 50, 100, 250$), and degrees of sparsity. Table \ref{tab:sim1_results} reports Accuracy$_{m,nz}$, Accuracy$_{m,z}$, and computational efficiency across privacy levels and sparsity settings. The \emph{BayesVFLReg} posterior $\pi_m(\bC \mid \bY_{\boldsymbol \Phi}, \bX_{\boldsymbol \Phi})$ closely approximates the centralized posterior $\pi(\bC \mid \bY, \bX)$, with average accuracies near 1 in all scenarios. The computational efficiency metric shows that higher privacy (smaller $m$) yields substantially more efficient sampling than the full posterior, a benefit that becomes increasingly important as $P$ grows. The computational efficiency for \emph{BayesVFLReg} is found to be a few folds higher than the gold standard centralized posterior.

\begin{table}[H]
\centering
\begin{tabular}{|c|c|c|c|c|c|}
  \hline
 & & $s$ = 10 & $s = 30$ & $s$ = 50 \\ 
  \hline
 \multirow{3}{*}{Avg. Accuracy$_{nz}$} & $m = 50$ & 0.995 & 0.997 & 0.998 \\
  & $m = 100$ & 0.994 & 0.997 & 0.998 \\ 
  & $m = 250$ & 0.994 & 0.996 & 0.997 \\ 
  \hline
  \multirow{3}{*}{Avg. Accuracy$_z$} & $m = 50$ & 0.997 & 0.997 & 0.998 \\
  & $m = 100$ & 0.998 & 0.998 & 0.999 \\ 
  & $m = 250$ & 0.999 & 0.999 & 0.999 \\ 
  \hline
  \multirow{4}{*}{Comp. Efficiency} & $m = 50$ & 9.880 & 9.846 & 9.917 \\ 
  & $m = 100$ & 7.748 & 7.463 & 7.680 \\ 
  & $m = 250$ & 5.562 & 4.540 & 5.542 \\ 
  & centralized & 3.261 & 3.228 & 2.932 \\ 
   \hline
\end{tabular}
\caption{The top six rows depict the measure of accuracy of estimating the full posterior of $\bC$ from compressed data, as expressed in \ref{eq:accuracy}. Results are disaggregated by zero and non-zero features, varying sparsity levels, and different values of $m=50,100,250$ representing differing privacy levels. This metric is averaged over 50 replicates; values closer to the upper bound of 1 indicate better posterior estimates with compressed data. These results demonstrate that the compressed approach closely approximates the centralized posterior. The bottom 4 rows present computational efficiency (see \ref{eq:comp_eff}), also averaged over 50 replicates, illustrating that data compression substantially improves computational performance relative to uncompressed methods.}
\label{tab:sim1_results}
\end{table}

\subsection{Simulation 2: Global Data Sketch versus Competitors}

In \textbf{Simulation 2}, we simulate our data using $n = 1000$, $P = 1000$ and $Q = 5$. We follow the same data generation process as outlined in \textbf{Simulation 1}. However, we generate features by varying correlations among features ($\rho$) to assess performance with changing feature correlations. The scenarios are as follows: 

\noindent \textit{Scenario 1:} Independent features with $\rho = 0$.

\noindent \textit{Scenario 2:} Moderately correlated features with $\rho = 0.5$.

\noindent \textit{Scenario 3:} Strongly correlated features with $\rho = 0.7$.

\noindent 
Both $\boldsymbol{A}$ and $\boldsymbol{B}$ are generated as in \textbf{Simulation 1}, with varying levels of sparsity set at $s=10,30,50$. Once simulated, the data are randomly partitioned across $K=3$ separate sites, under the assumption that no communication occurs between sites. This simulation procedure is replicated 50 times for each combination of correlation scenario and sparsity level.


We choose $m = 200$ to keep the normalized mutual information below 0.0015 bits at each site between $\bX$ and $\bX_\Phi$ for  \textit{Scenarios 1, 2, \& 3}. To assess the impact of sketching on posterior estimation of $\bC$, we implement three competing estimators. As a \emph{BayesVFLReg} competitor, we use \emph{dRidge}, a distributed ridge regression approach proposed by \cite{gratton2018distributed}, which we fit separately to each response due to its univariate regression design. For non-VFL methods operating on the full, centralized data, we implement the variational Bayes approach \emph{locus} \citep{ruffieux2017efficient}, and the frequentist sparse reduced-rank regression method, \emph{srrr} \citep{chen2012sparse}. In the absence of a centralized posterior gold standard, \emph{locus} and \emph{srrr} serve as benchmarks for non-VFL performance, enabling a comparative assessment with \emph{BayesVFLReg}.

We evaluate inferential accuracy using the MSE and $\text{MSE}_{nz}$ metrics defined as in \textbf{Simulation 1}. We also examine the coverage of 95\% credible or confidence intervals, both overall and for the non-zero coefficients (coverage$_{nz}$). For \emph{dRidge}, we implement bootstrap-based repeated sampling for uncertainty quantification, which incurs substantially higher computational cost. For the frequentist \emph{srrr} competitor, we do not implement repeated sampling and therefore do not report uncertainty measures for \emph{srrr}. Finally, we apply the variable selection procedures described in Algorithm~\ref{alg:feature_classification} and compute the false discovery rate (FDR) for each method to evaluate its ability to identify important features.

\subsubsection{Simulation 2 Results}

Figure \ref{fig:MSE_plot_sim2} compares the MSE results  for all zero and non-zero values of $\bC$. We observe that the two VFL methods, \emph{BayesVFLReg} and \emph{dRidge}, achieve MSE performance comparable to the non-VFL competitors when correlation is low. As correlation among covariates increases, the \emph{BayesVFLReg} method increasingly outperforms \emph{locus}, while being comparable to the other competing methods, particularly for estimating nonzero coefficients (Figure \ref{fig:MSE_plot_sim2}). In high-correlation, low-sparsity settings, \emph{BayesVFLReg} provides notably more accurate estimates of the non-zero coefficients than the other VFL competitor \emph{dRidge}.

Figure \ref{fig:coverage_plt} presents coverage results across different correlation scenarios and sparsity levels. Overall, all methods achieve coverage close to one. However, consistent with prior literature on high-dimensional regression \citep{guhaniyogi2025bayesian}, all methods exhibit under-coverage when estimating non-zero coefficients. Notably, \emph{BayesVFLReg} provides substantially better coverage$_{nz}$ for non-zero coefficients compared to \emph{dRidge}, which shows pronounced under-coverage and much shorter 95\% credible intervals. The centralized method \emph{locus}, which uses variational approximation, also greatly underestimates the coverage for non-zero coefficients, yielding the narrowest intervals. For all methods, under-coverage of non-zero coefficients becomes more pronounced as correlation increases and sparsity decreases.

\begin{figure}[ht]
    \centering
    \includegraphics[width=1\linewidth]{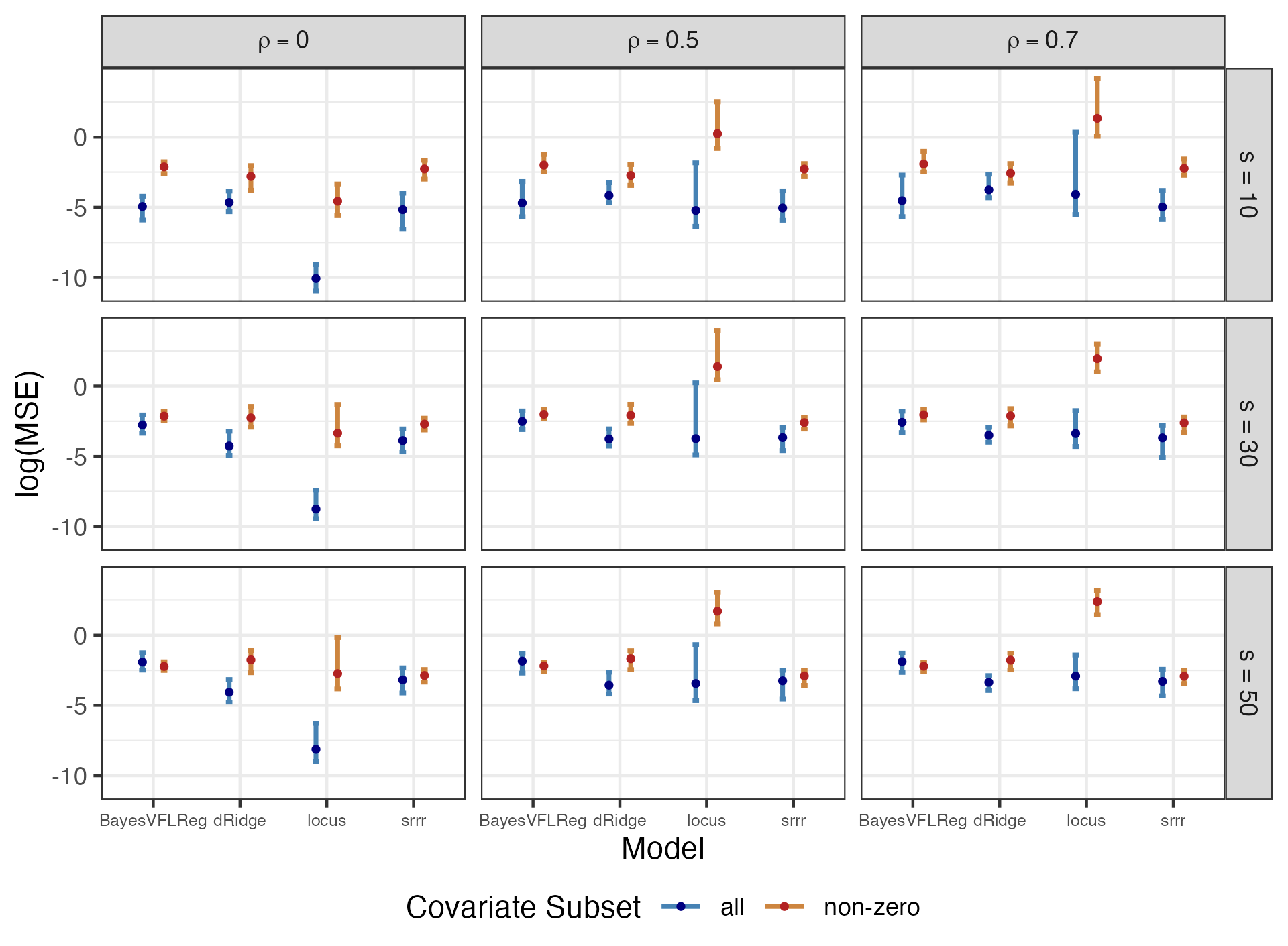}
     
    \caption{Figure depicts the MSE (blue) and MSE$_{nz}$ (red) of the four respective models; \emph{BayesVFLReg}, \emph{dRidge}, \emph{locus} and \emph{srrr} on the log scale. The points depict the average log(MSE) or log(MSE$_{nz}$) across replicates while the bars depict a 95\% interval of the replicates. 
    It can be observed that, as the correlation increases, the data sketching approach exhibits progressively better comparative performance.}
    \label{fig:MSE_plot_sim2}
\end{figure}

\begin{figure}[ht]
    \centering
    \includegraphics[width=1\linewidth]{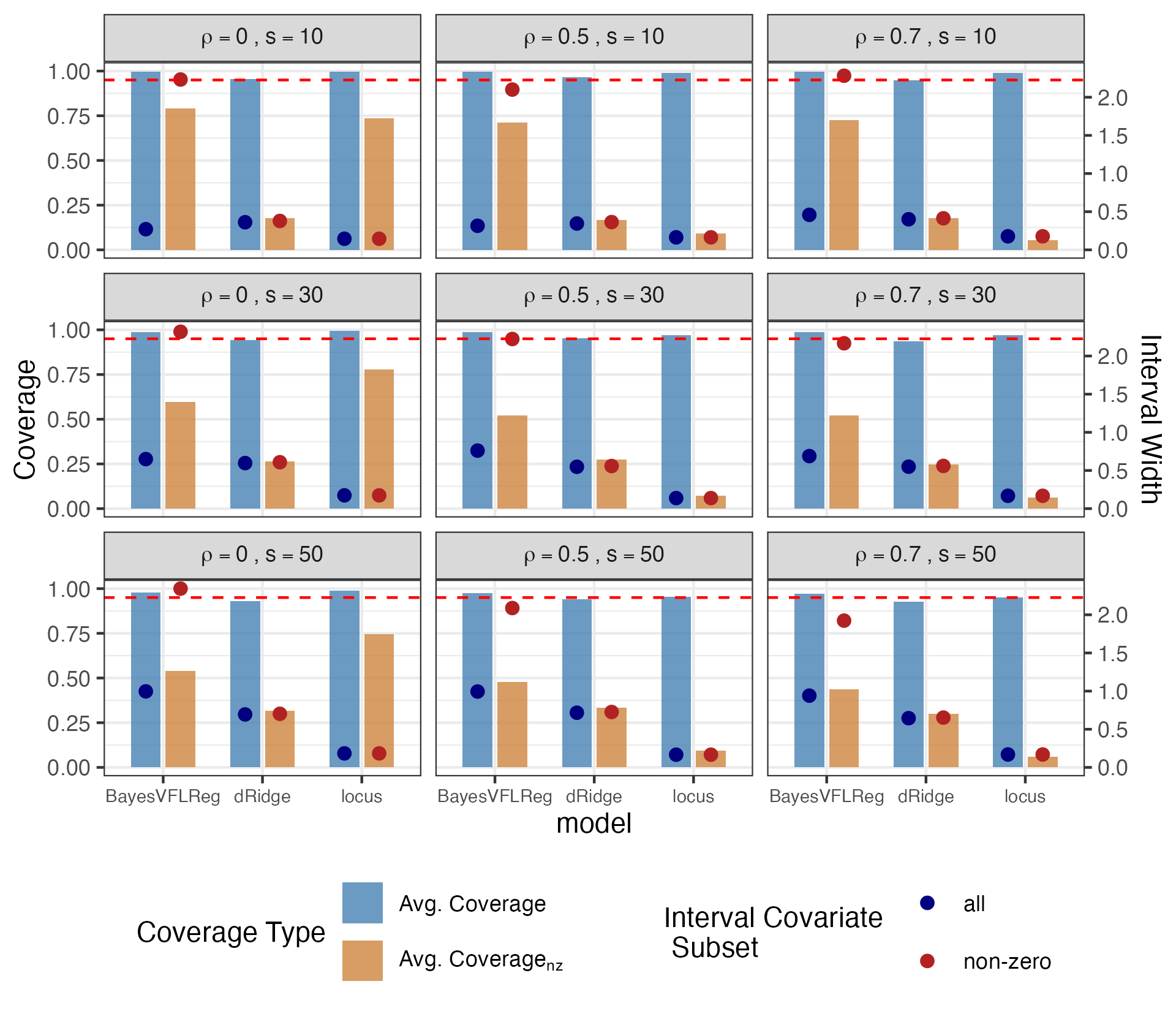}
  \vspace{-.7cm}
    \caption{Figure displays average coverage (blue) and average coverage$_{nz}$ (red) as bars, along with points denoting the average interval width under each simulation case. This demonstrates that \emph{BayesVFLReg} achieves superior overall coverage; however, its interval widths are larger than those of the other two methods, which, in turn, exhibit notably poorer performance for coverage$_{nz}$.}
    \label{fig:coverage_plt}
\end{figure}

\begin{table}[H]
\begin{adjustwidth}{-1in}{-1in}
\centering
\begin{tabular}{|c|c|c|c|c|c|c|c|c|c|}
\hline
   & \multicolumn{3}{|c|}{\textit{Scenario 1}} & \multicolumn{3}{c|}{\textit{Scenario 2}} & \multicolumn{3}{c|}{\textit{Scenario 3}} \\
  \cline{2-10}
   & $s = 10$ & $s = 30$ & $s = 50$ & $s = 10$ & $s = 30$ & $s = 50$ & $s = 10$ & $s = 30$ & $s = 50$ \\ 
  \hline
    BayesVFLReg & 0.004 & 0.037 & 0.063 & 0.074 & 0.117 & 0.159 & 0.132 & 0.079 & 0.202 \\ 
  dRidge & 0.000 & 0.000 & 0.000 & 0.000 & 0.000 & 0.000 & 0.000 & 0.000 & 0.000 \\ 
  locus & 0.789 & 0.561 & 0.435 & 0.797 & 0.575 & 0.464 & 0.804 & 0.586 & 0.471 \\ 
  srrr & 0.349 & 0.241 & 0.134 & 0.391 & 0.343 & 0.252 & 0.316 & 0.288 & 0.271 \\
  \hline
\end{tabular}
    \end{adjustwidth}
    \caption{Table reports the \textbf{average false discovery rate (FDR)}. Given $\tilde{\alpha} = 0.05$, this value should be close to 0.05. While \emph{dRidge} successfully separates non-zero from zero coefficients, as indicated by the coverage, it struggles to accurately estimate their magnitude.}
    \label{tab:sim2_results}
\end{table}

Table \ref{tab:sim2_results} depicts the FDR of the post processing procedure to identify important features. \emph{BayesVFLReg} outperforms each method in terms of FDR except  \emph{dRidge}, which appears to be the best performer. Taken together, these results indicate that \emph{dRidge} is effective at separating zero and non-zero features (very low FDR) but provides inferior uncertainty quantification, severely underestimating uncertainty of non-zero effects. The \emph{BayesVFLReg}, in contrast, delivers a more favorable trade-off: it achieves competitive MSE, higher coverage of non-zero coefficients, and reasonably low FDR across all sparsity and correlation scenarios, making it the most reliable option in settings with highly correlated, predominantly null features.

\FloatBarrier

\section{Study of Neuroimaging Data on Aging}
\label{sec:data_app}
We apply the Bayesian VFL Regression (\emph{BayesVFLReg}) to multi-modal human neuroimaging data collected first-hand by co-author Dr. Bernard, at the Lifespan Cognitive and Motor Neuroimaging Laboratory at Texas A\&M University. The original study recruited 138 healthy adults living independently aged 35 to 86 years. After excluding participants due to incomplete or noisy data, the final sample for our analysis consisted of 126 individuals (mean age: 57 years; 54\% female).
The participants underwent structural and resting-state magnetic resonance imaging (MRI), and data preprocessing was carried out according to \cite{Hicks_2023} and \cite{Ballard_2022}. Following the network definitions of \cite{Jackson_2023}, an initial set of 75 regions of interest (ROIs) were defined based on 6 key networks: the default mode, frontal-parietal, control, emotion, motor, and cerebellar-basal ganglia.  After excluding 6 ROIs for alignment with structural data, the final analysis used a set of 69 cortical and subcortical ROIs.

For each participant, a $69 \times 69$ correlation matrix was generated by cross-correlating the time series extracted from 69 regions of interest (ROIs). This analysis was performed in the CONN toolbox (version 21a) \citep{WhitfieldGabrieli2012ConnAF}. The upper triangle of each matrix was then vectorized, resulting in $2,346$ connectivity predictors for each subject. In addition, participants completed several behavioral measures: the Montreal Cognitive Assessment (\texttt{MoCA}) as a global cognitive screening tool; the Purdue Pegboard task, yielding bi-manual and unilateral dexterity scores (\texttt{Pegboard.Both}, \texttt{Avg.Right}, \texttt{Avg.Left}), a grip strength test, summarized as average left and right hand strength (\texttt{Avg.Left}, \texttt{Avg.Right}), and a Digit Symbol Coding task (\texttt{Digit.Symbol.Coding}) measuring processing speed, attention, and psychomotor efficiency. Collectively, these measures provide a cognitive-motor profile that serve as a behavioral basis for examining individual variations in functional connectivity.

Under a vertical federated learning framework, we assume that the dataset was vertically partitioned across two secure, independent sites: one hosting the MRI imaging data and the other holding the corresponding cognitive and motor test results for the same participant cohort. In compliance with data privacy protocols, raw data was not exchanged between the sites.

Consistent with a large body of neuroscientific literature, we hypothesized a sparse relationship between functional brain connectivity and behavior \citep{Varoquaux2013, Smith2009}. This implies that while the number of potential predictors (i.e., network edges) is large, a small subset of these connections should be meaningfully associated with participants' cognitive and motor abilities.
In the absence of ground-truth for the regression coefficient matrix ($\bC$) for real-world data, we evaluate our proposed \emph{BayesVFLReg} method by benchmarking it against the competitors, distributed ridge regression (\emph{dRidge}) \citep{gratton2018distributed}, \emph{locus} \citep{ruffieux2017efficient}, and sparse reduced rank regression (\emph{srrr}) \citep{chen2012sparse}. Our evaluation is designed to assess the model's predictive accuracy. To compare how well each model fit the out-of-sample behavioral data ($\bY$), we compute the ratio of their respective Mean Squared Errors (MSE):
    \[
        \text{MSE}_\bY = \frac{\|\text{vec}(\bY) - \text{vec}(\widehat{\bC}\bX) \|^2}{\|\text{vec}(\bY) - \text{vec}(\widetilde{\bC}\bX)||^2}
    \]
Where $\widetilde{\bC}$ represents the estimated coefficient matrix of the competitor and $\widehat{\bC}$ is the posterior mean from \emph{BayesVFLReg}. This ratio directly measures the predictive accuracy of our model relative to the competitors as a benchmark for each choice of $m$. A value less than one indicates a superior predictive fit for \emph{BayesVFLReg}. We present these results for $m \in \{20, 30, 50, 75, 100\}$ in Figure~\ref{fig:app_results}.

\begin{figure}[ht]
    \centering
        \includegraphics[width=.6\textwidth]{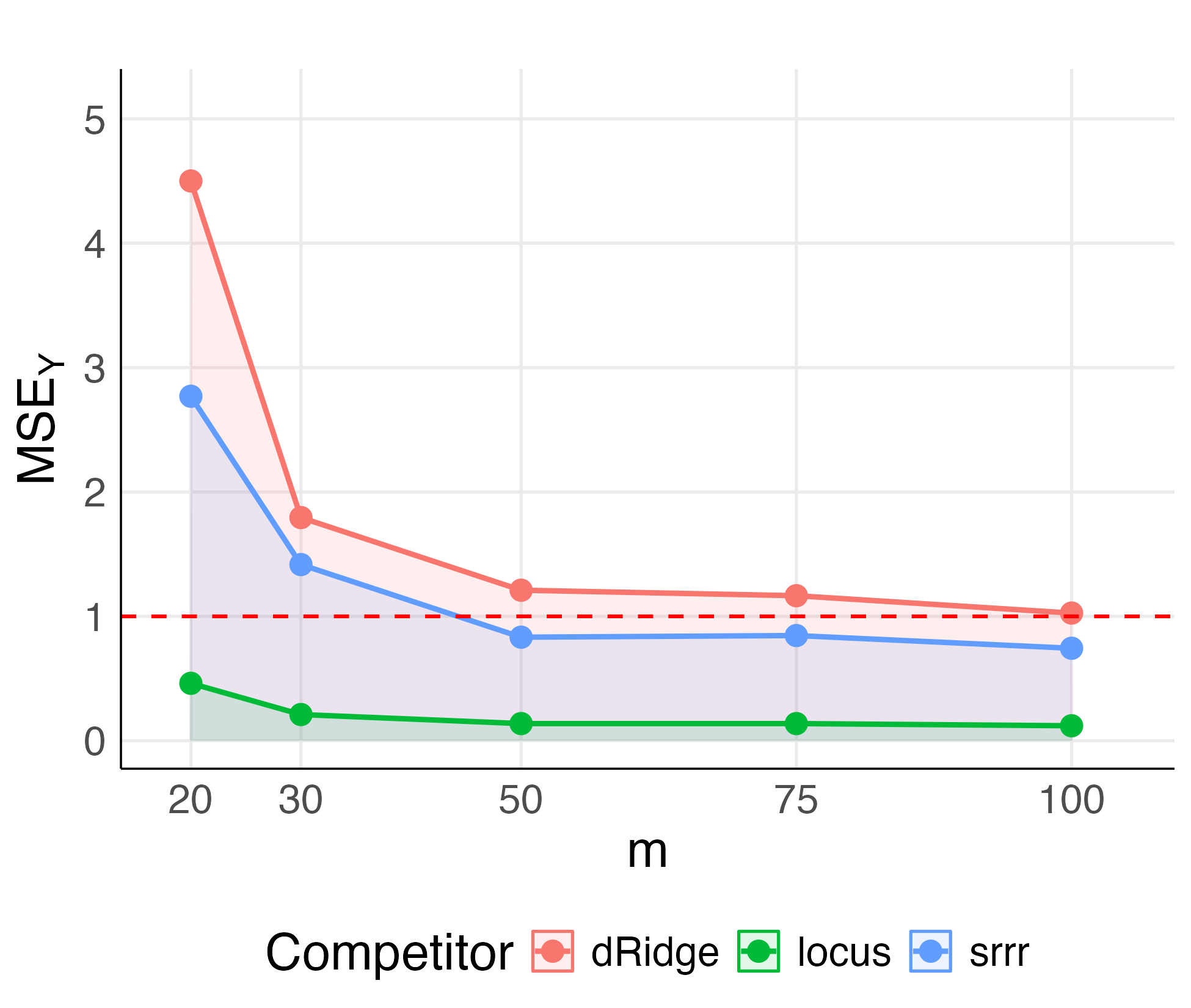} 
    \caption{
    Shows relative predictive performance, measured as the ratio of the Mean Squared Error (MSE$_\bY$) of \emph{BayesVFLReg} to that of each respective competitor. The horizontal reference line at 1.0 indicates performance parity with each competitor; values below this line demonstrate that \emph{BayesVFLReg} outperforms the competitor. }
    \label{fig:app_results}
\end{figure}

Based on the results presented in Figure \ref{fig:app_results}, $m=50$ is selected  for subsequent post-processing. This value is chosen as yields a near-minimal predictive Mean Squared Error (MSE$_\bY$) comparatively across competitors.


\begin{figure}[ht]
    \centering
    \begin{subfigure}[b]{0.39\textwidth}
        \centering
        \includegraphics[width=\textwidth]{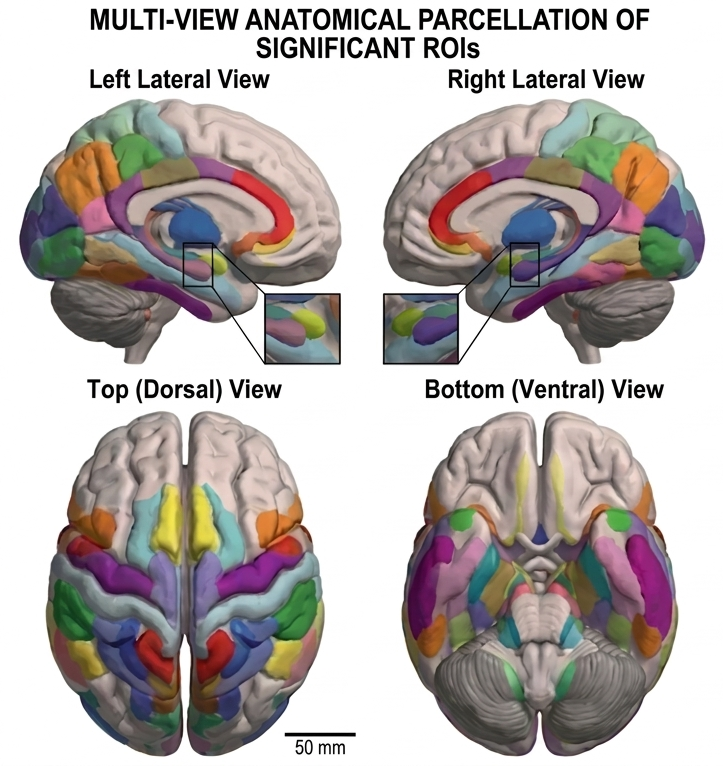} 
        \caption{Brain Images with Highlighted ROIs}
        \label{fig:brain_img}
    \end{subfigure}
    \hfill 
    \begin{subfigure}[b]{0.59\textwidth}
        \centering
        \includegraphics[width=\textwidth]{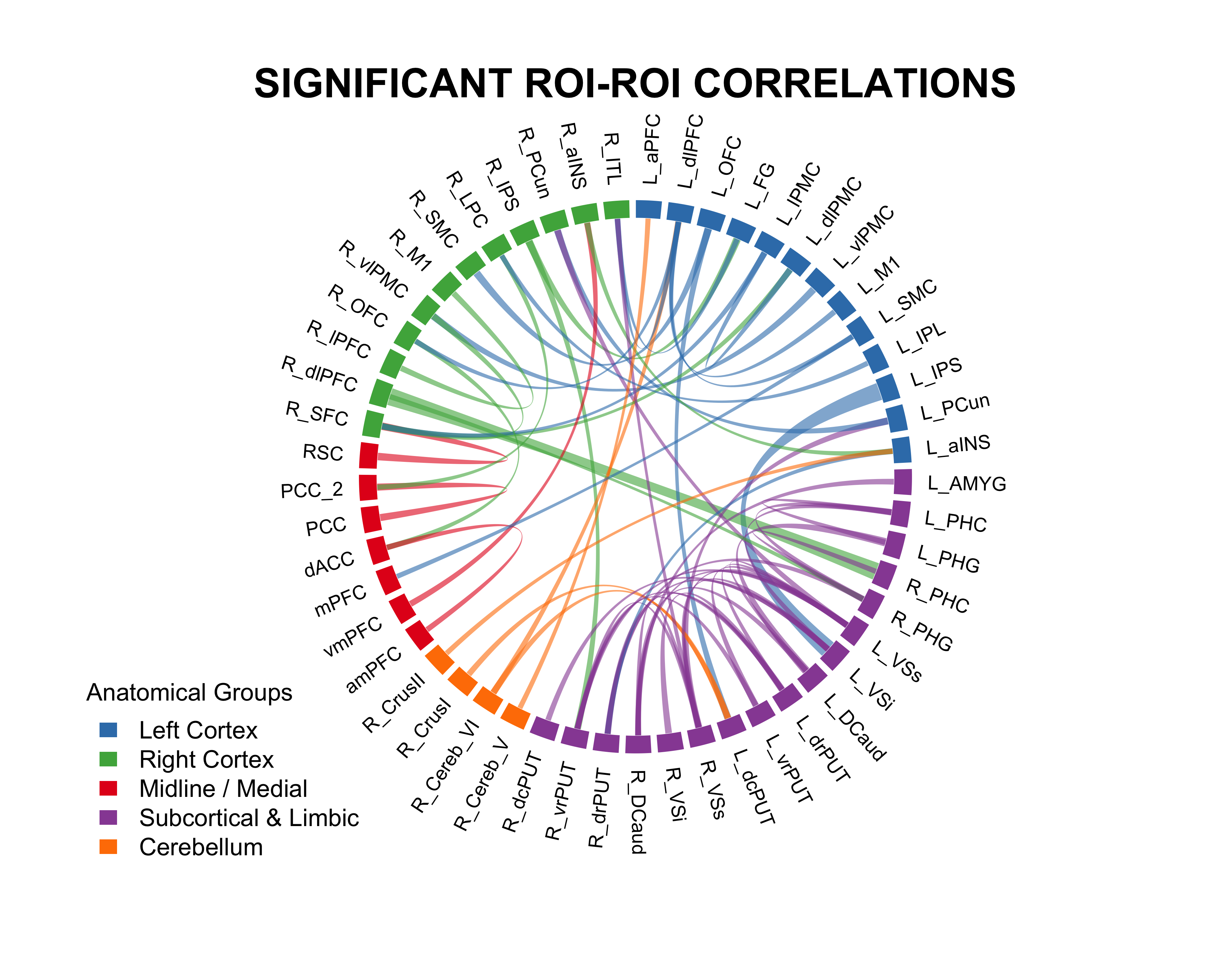} 
        \caption{Significant brain connections}
        \label{fig:connectogram}
    \end{subfigure}
    \caption{The significant ROI-ROI correlations on the participants' cognitive and motor functions are depicted here. Figure (\ref{fig:brain_img}) depicts the brain highlighting the ROIs that were found among the 60 significant links. Figure (\ref{fig:connectogram}) illustrates the significant ROI-ROI correlations with the edge width depicting the absolute strength of the relationship between the correlation and the cognitive and motor functions of the participants found by taking the row means of the posterior mean of $\bC$. A full list of significant ROI with their shorthand mappings can be found in Table \ref{tab:roi_key} from the Appendix.
    }
    \label{fig:brain_connect}
\end{figure}

Our analysis of 2,346 potential network connections identified a key subset of $60$ links that significantly influenced participants' cognitive and motor functions, these $60$ links are depicted within Figure~\ref{fig:brain_connect}. Across the set of significant ROI-ROI correlations, a consistent pattern emerged in large-scale networks that plausibly underlie both cognitive and motor outcomes. Many effects involved frontoparietal and default-mode nodes, including the left inferior parietal lobule, lateral parietal cortex, posterior cingulate/retrosplenial regions, and mesial prefrontal cortex. These regions are central to multimodal integration, attention, episodic memory, and higher-order cognitive control, suggesting that altered coupling within and between parietal and midline default-mode structures is associated with patients’ cognitive performance. Parahippocampal connections further implicate medial temporal-default-mode circuitry, consistent with effects on memory and global cognition.

A second prominent theme was the involvement of cortico-striato-motor loops. The superior ventral striatum, dorsal caudate, and ventral/dorsal rostral putamen appeared frequently, with significant correlations to cingulate (medial, supragenual, dorsal anterior), anterior insula, mesial prefrontal cortex, and multiple premotor and primary sensorimotor ROIs (e.g., left lateral premotor cortex, dorsolateral premotor cortex, primary sensorimotor cortex). This pattern is characteristic of basal ganglia circuits supporting motor control, action selection, and reward-based decision-making, and aligns well with the observed relationships to motor measures such as pegboard performance. ROIs that recur across many significant edges include the left dorsal caudate, left ventral/dorsal rostral putamen, right superior ventral striatum, medial/dorsal anterior cingulate cortex, left lateral/dorsolateral premotor cortex, inferior/lateral parietal cortex, and posterior/retrosplenial cingulate, highlighting these regions as key hubs where connectivity differences are most strongly linked to patients’ cognitive and motor abilities.

\FloatBarrier

\section{Theoretical Guarantees for BayesVFLReg}\label{sec:VFLreg_theory}
\subsection{Motivation and Overview}
This section establishes theoretical gurantee for \emph{BayesVFLReg} framework described in Section~\ref{sec:VFL_MVR}. For algebraic simplicity, we set $\sigma_1^2=\cdots=\sigma_Q^2$ and set them at $1$, without loss of generality. While Section~\ref{sec:VFL_MVR} performs full posterior computation using the standard posterior, our theoretical analysis, however, is conducted under the $\alpha$-fractional posterior framework with $\alpha\in(0,1)$, following popular approaches \citep{chakraborty2020bayesian}. The adoption of the fractional posterior is primarily a technical convenience: it requires only a prior mass condition on Kullback-Leibler (KL) neighborhoods of the truth to establish posterior concentration, whereas the standard posterior additionally requires the construction of sieves and exponentially consistent test functions, which are difficult to verify for heavy-tailed shrinkage priors. In this section, we further show how convergence of $\alpha$-fractional posterior to  the full posterior when $\alpha\rightarrow 1$. Consequently, any concentration result established for the fractional posterior at $\alpha\in(0,1)$ carries over to the standard posterior used in \emph{BayesVFLReg}.

\subsection{Notations and Framework}
\subsubsection{Divergence Measures}
We collect here the divergence measures used in the theoretical analysis. For two 
probability densities $p$ and $q$ defined with respect to a common dominating 
measure $\mu$:

\begin{itemize}
    \item \textbf{Hellinger distance:} $ h^2(p, q) = \int \bigl(\sqrt{p} - \sqrt{q}\bigr)^2 \, d\mu.$

     \item \textbf{Total Variation distance:} $ TV(p, q) = \int \bigl |p - q| \, d\mu.$

    \item \textbf{Kullback--Leibler (KL) divergence:} $\mathrm{KL}(p, q) = \int p \log(p/q)\, d\mu.$

    \item \textbf{R\'{e}nyi divergence} (of order $\alpha \in (0,1)$):
    $D_\alpha(p, q)
        = \frac{1}{\alpha - 1}
          \log \int p^\alpha q^{1-\alpha} \, d\mu.$

    \item \textbf{$\alpha$-affinity:}
    $A_\alpha(p, q)
        = \int p^\alpha q^{1-\alpha} \, d\mu
        = e^{-(1-\alpha) D_\alpha(p,q)}.$
        
    When $\alpha = 1/2$, $A_{1/2}(p,q)$ coincides with the Hellinger affinity.
    These divergences satisfy the equivalence relations
    \[
        D_{1/2}(p,q) \geq 2h^2(p,q),
        \qquad
        \frac{\alpha}{\beta}\cdot\frac{1-\beta}{1-\alpha} D_\beta
        \leq D_\alpha \leq D_\beta,
        \quad 0 < \alpha \leq \beta < 1.
    \]
\end{itemize}
Matrices are denoted by bold uppercase letters; $\|\cdot\|_F$ and $\|\cdot\|_2$ denote the Frobenius and $l_2$ norms respectively; $\lambda_{\max}(\bH)$ denotes the largest eigenvalue of a square matrix $\bH$.
\subsubsection{Data Generating Model}
Recall from Section~\ref{sec:VFL_MVR} that the true data-generating model is
\begin{equation}\label{eq:true_model_th}
    \bY = \bX\bC^* + \bE, 
    \qquad 
    {\boldsymbol e}_i \overset{\mathrm{iid}}{\sim} \mathcal{N}_Q(\bzero, \bI_Q),
    \tag{1}
\end{equation}
where $\bC^* \in \mathbb{R}^{P \times Q}$ is the true coefficient matrix generating the data. Similar to fitted model, we assumed unit variance for the true data generating model. Pre-multiplying both sides of equation~(\ref{eq:true_model_th}) by the shared sketching 
matrix $\bPhi \in \mathbb{R}^{m \times n}$ yields 
the \emph{true sketched model}:
\begin{equation}
    \bY_{\bPhi}= \bX_{\bPhi}\bC^* + \bPhi \bE, 
    \qquad 
    \bPhi \bE \sim \mathrm{MatrixNormal}(\bzero,\, \bPhi\bPhi^T,\, \bI_Q),
    \tag{2}
\end{equation}
This acts as the true data generating sketched model in our theoretical analysis. Let the density of our fitted model with regression coefficient $\bC$ is denoted 
$p^{(m)}(\bY_{\bPhi}\mid \bX_{\bPhi}, \bC)$, and the density of the true data generating model with the true regression coefficient $\bC^*$ is denoted $p^{(m)}(\bY_{\bPhi}\mid \bX_{\bPhi},\bC^*)$.

\subsection{Assumptions}
\label{sec:assumptions}


\medskip

\noindent\textbf{Assumption 1 (Fixed response dimension).}
The number of response variables $Q$ is fixed as $n \to \infty$.

\medskip

\noindent\textbf{Assumption 2 (Sparse low-rank true coefficient).}
The true coefficient matrix admits the decomposition $\bC^* = \bB^* \bA^{*T},$
where $\bB^* \in \mathbb{R}^{P \times R^*}$, $\bA^* \in \mathbb{R}^{Q \times R^*}$, and $R^* = \kappa Q\quad\text{for some}\quad\kappa \in \{1/Q, 2/Q, \ldots, 1\}.$ The matrix $\bA^*$ is semi-orthogonal, $(\bA^*)^T \bA^* = I_{R^*}$, and all but $s^*$ rows of $\bB^*$ are identically zero, where $s^* = o(P)$. Furthermore,
$\max_{j,h} \bigl| C^*_{j,h} \bigr| < T$ for some $T > 0$.

\medskip

\noindent\textbf{Assumption 3 (Bounded design matrix).}
For $\bX_j$ the $j$th column of $\bX$, $\max_{1 \leq j \leq P} \|\bX_j\|^2 = O(n).$

\medskip

\noindent\textbf{Assumption 4 (Sketching dimension).}
The sketching dimension $m$ satisfies $s^* \log(P) = o(m),\: m = o(n).$

\medskip

\noindent\textbf{Assumption 5 (Concentration of sketching matrix).}
For some positive constants $c', c'' > 0$, $\|\bPhi\bPhi^T - I_m\|_2 \leq c' \sqrt{m/n},$
\medskip

\noindent\textbf{Remark:} Assumption 5 holds with probability at least $1 - e^{-c' m}$ over the construction of $\bPhi$ in this article. This holds almost surely for Gaussian sketching matrices \cite{vershynin2011spectral, guhaniyogi2025sketching}.

\noindent\textbf{Remark.}
Assumption~5 connects the sketching dimension to the intrinsic complexity of the true model: choosing
\[
m \asymp s^* \log P
\]
(up to constants) is the minimal sketching dimension needed to preserve the sparse low-rank structure of $\bC^*$. Assumption~2 mirrors the structure exploited by the shrinkage prior on $\bC$ in Section~2.2.

\bigskip

\subsection{Fractional Posterior and Contraction Rate}
\label{sec:fractional_posterior_rate}

For $\alpha \in (0,1)$, the $\alpha$-fractional posterior for $\bC$ under the approximating sketched model is
\begin{equation}
\Pi_\alpha(\bC \mid \tilde{\bY}, \bX_{\bPhi})
\;\propto\;
\left\{
p^{(m)}(\bY_{\bPhi}\mid \bX_{\bPhi}, \bC)
\right\}^\alpha
\pi(\bC),
\tag{3}
\end{equation}
where $\pi(\bC)$ is the prior on $\bC = \bB\bA^T$ induced by the shrinkage prior on $\bB$ and the standard Gaussian prior on $\bA$ specified in Section~2.2. The contraction rates derived below for $\Pi_\alpha$ with $\alpha \in (0,1)$ carry over to the standard posterior $\Pi_1$ used in \emph{BayesVFLReg}, as we argue in this section through Theorem~\ref{TV_thm}.

Define the \emph{contraction rate} for the sketched model as
\begin{equation}
\epsilon_m^2
=
\frac{Q R^* + R^* s^* \log P}{m}.
\tag{4}
\end{equation}
This is the analogue of the minimax rate $\frac{Q R^* + R^* s^* \log P}{n}$
for the full-data model~\cite{bunea2012joint}, with the effective sample size replaced by the sketching dimension $m$. Assumption~5 ensures $m \epsilon_m^2 \to \infty$, which is necessary for the posterior to concentrate.

\subsection{Main Results}
We state the three main theoretical results. All proofs are deferred to the Appendix.

\medskip

\begin{lemma}[Prior concentration in the sketched model]
\label{lem:prior_conc_main}
Define the KL neighborhood of the true coefficient matrix at radius $\epsilon_m$ as
\[
B_m^*(\bC^*, \epsilon_m)
=
\left\{
\bC :
\mathrm{KL}\!\left(
p^{(m)}(\bY_{\bPhi}\mid \bX_{\bPhi}, \bC^*),\,
p^{(m)}(\bY_{\bPhi}\mid \bX_{\bPhi}, \bC)
\right)
\leq m\epsilon_m^2
\right\}.
\]
Under Assumptions 1-5, there exists a positive constant $K$ such that, 
\[
\pi\!\left(B_m^*(\bC^*, \epsilon_m)\right) \geq e^{-K m\epsilon_m^2}.
\]
\end{lemma}
Lemma~\ref{lem:prior_conc_main} establishes that the shrinkage prior induced by the \emph{BayesVFLReg} prior specification assigns exponentially large mass to Kullback–Leibler neighborhoods of the truth
in the sketched model. 

\medskip\medskip

\begin{lemma}[Lower bound for the fractional posterior denominator]
\label{thm:Dm_lower_main}
Let
\[
D_m
=
\int e^{-\alpha\, r(\bC, \bC^*)}\, \pi(\bC)\, d\bC,
\qquad
r(\bC, \bC^*)
=
\log\frac{
p^{(m)}(\bY_{\bPhi}\mid \bX_{\bPhi}, \bC^*)
}{
p^{(m)}(\bY_{\bPhi}\mid \bX_{\bPhi}, \bC)
}.
\]
Under Assumptions 1-5, for any $D > 1$ and $t > 0$, with $P_{\bC^*}^{(m)}$-probability at least
$1 - K'/\{(D-1+t)^2 m\epsilon_m^2\}$,
\[
D_m \;\geq\; e^{-(D+t)m\epsilon_m^2},
\]
for some positive constant $K'>0$.
\end{lemma}

\noindent
Theorem~\ref{thm:Dm_lower_main} provides the high-probability lower bound on the normalizing constant of the fractional posterior that, combined with the upper bound on the numerator derived from Fubini's theorem and Markov's inequality, yields Theorem~\ref{thm:renyi_main} below.

\begin{theorem}[Fractional posterior concentration under the R\'enyi divergence]
\label{thm:renyi_main}
Let
\[
U_m
=
\left\{
\bC :
D_\alpha\!\left(
p^{(m)}(\bY_{\bPhi}\mid \bX_{\bPhi}, \bC^*),\,
p^{(m)}(\bY_{\bPhi}\mid \bX_{\bPhi}, \bC)
\right)
\geq
\frac{D + 3t}{1 - \alpha}\, m\epsilon_m^2
\right\}.
\]
Under Assumptions 1–5, for any $D \geq 1$, $t > 0$ and some positive constant $K'$, with $P_{\bC^*}^{(m)}$-probability at least
$1 - K'/\{(D-1+t)^2 m\epsilon_m^2\}$,
\[
\Pi_\alpha\!\left(U_m \mid \tilde{\bY}\right) \;\leq\; e^{-t m\epsilon_m^2}.
\]
\end{theorem}

\medskip\medskip

\begin{theorem}[Fractional posterior concentration under the Hellinger distance]
\label{thm:hellinger_main}
Under Assumptions 1–5, with $P_{\bC^*}^{(m)}$-probability at least
$1 - K'/\{(D-1+t)^2 m\epsilon_m^2\}$,
\[
\Pi_\alpha\!\left(
\left\{
\bC :
\frac{1}{m}h^2\!\left(
p^{(m)}(\bY_{\bPhi}\mid \bX_{\bPhi}, \bC^*),\,
p^{(m)}(\bY_{\bPhi}\mid \bX_{\bPhi}, \bC)
\right)
\geq
\frac{D + 3t}{2(1-\alpha)}\, \epsilon_m^2
\right\}
\Bigg\lvert\, \tilde{\bY}
\right)
\;\leq\;
e^{-t m\epsilon_m^2},
\]
for any $D \geq 1$, $t \geq 0$, and some positive constant $K'$ for large $n$.
\end{theorem}

\noindent
\textbf{Remark.} Theorem~\ref{thm:hellinger_main} follows immediately from Theorem~\ref{thm:renyi_main} via the equivalence between R'enyi divergences and the Hellinger distance established in Section~5.2.1. Together, Theorems~\ref{thm:renyi_main} and~\ref{thm:hellinger_main} provide non-asymptotic bounds on the fractional posterior probability that the fitted density of the sketched data deviates from the true density by more than a threshold governed by
\[
\epsilon_m^2
=
\frac{Q R^* + R^* s^* \log P}{m},
\]
Notably, this threshold shares the same form as the minimax risk for sparse low-rank regression in the full-data model \cite{bunea2012joint}, with the full sample size $n$ replaced by the sketching dimension $m$, reflecting the statistical cost of privacy-preserving data compression.\\
\noindent\textbf{Remark.}
For analytical convenience, Assumption~5 is treated as holding
with probability~1 throughout the proofs. In practice, however,
it holds with probability at least $1 - e^{-c'm}$ over the draw
of $\bPhi$. To make the non-asymptotic bound explicit under this
probabilistic guarantee, define the event
\[
\mathcal{A}
= \left\{
\bPhi : \|\bPhi\bPhi^T - \bI_m\|_F \leq c''\sqrt{m/n}
\right\}
\]
and observe that
\[
\Pi_\alpha(U_m \mid \tilde{\bY})
\;\leq\;
\Pi_\alpha(U_m \mid \tilde{\bY},\, \mathcal{A})
+ P(\mathcal{A}^c)
\;\leq\;
e^{-tm\epsilon_m^2} + e^{-c'm},
\]
so the overall bound retains the same exponential form, with an
additional term that is negligible for large $m$.

Since the non-asymptotic bound in
Theorems~\ref{thm:renyi_main} and~\ref{thm:hellinger_main}
is established under the $\alpha$-fractional posterior, it is
natural to ask how closely the fractional posterior approximates
the standard posterior corresponding to $\alpha = 1$. We address
this in the following result.
\begin{theorem}\label{TV_thm}
Let $\Pi_\alpha(\bC|\tilde{\bY})$ and  $\Pi(\bC|\tilde{\bY})$ be the $\alpha$-fractional posterior and ordinary full posterior (at $\alpha=1$). Then $\lim\limits_{\alpha\rightarrow 1}\mathrm{TV}(\Pi_\alpha(\bC|\tilde{\bY}),\Pi(\bC|\tilde{\bY}))=0$
\end{theorem}

\section{Conclusion}
\label{sec:conclusion}

We introduce the Bayesian Vertical Federated Learning (BayesVFLReg) framework tailored for high-dimensional multivariate sparse reduced-rank regression. While traditional horizontal approaches dominate the literature, our method tackles the more complex setting of vertically partitioned data distributed across multiple distinct sites. By leveraging data sketching techniques, our framework successfully constructs a centralized, global representation of the posterior distribution without compromising site-level privacy. Crucially, this approach enables principled uncertainty quantification, a feature largely missing from existing VFL methodologies.

Beyond its foundational architecture, we established the theoretical validity of our method by deriving non-asymptotic bounds for the posterior concentration around the true density. Our simulation studies further demonstrate that the proposed framework delivers competitive, and often superior, performance against both alternative VFL methods and standard algorithms built for centralized data. Ultimately, the privacy-compliant sketched data representation introduced here provides a versatile bridge to downstream models. This paves the way for promising future extensions into vertically federated generalized linear models, survival analysis, and complex non-parametric Bayesian learning.


\section*{Acknowledgments}

Dr. Rajarshi Guhaniyogi is supported by National Science Foundation DMS-2210672, and National Institute of Health R01NS131604.

Portions of this research were conducted with the advanced computing resources
provided by Texas A\&M Department of Statistics Arseven Computing Cluster.

\section*{Data and Software Availability}
An open-source implementation of the \emph{BayesVFLReg} algorithm, is publicly available on GitHub at \url{https://github.com/brighalvy/BayesVFLReg}.


\newpage

\appendix

\section{Full Conditionals}

We begin by vectorizing $\boldsymbol{\Phi}\boldsymbol{Y} = \boldsymbol{\Phi}\boldsymbol{X}\boldsymbol{C} + \boldsymbol{E}$ to obtain:\begin{equation}
    \boldsymbol{y} = \boldsymbol{\mu}_{\boldsymbol{y}} + \boldsymbol{e}
\end{equation}
\noindent Letting,
\begin{align*}
    \boldsymbol{y} &= \text{vec}((\boldsymbol{\Phi}\boldsymbol{Y})^T) \in \R^{mQ \times 1} \\
    \boldsymbol{\mu}_{\boldsymbol{y}} &= \text{vec}((\boldsymbol{\Phi}\boldsymbol{X}\boldsymbol{C})^T) = (\boldsymbol{\Phi}\boldsymbol{X} \otimes \boldsymbol{A})\boldsymbol{\beta} = (\Phi\boldsymbol{X}\boldsymbol{B} \otimes I_{Q})\boldsymbol{a} \in \R^{mQ \times 1} \\
    \boldsymbol{e} &= \text{vec}(\boldsymbol{E}^T) \sim N_{mQ}(\boldsymbol{0}, \tilde{\boldsymbol{\Sigma}}) \\
    \boldsymbol{\beta} &= \text{vec}(\boldsymbol{B}^T) \in \R^{PQ \times 1} \\
    \boldsymbol{a} &= \text{vec}(\boldsymbol{A}) \in \R^{Q^2 \times 1} \\
    \tilde{\boldsymbol{\Sigma}} &= \text{diag}(\boldsymbol{\Sigma}, ..., \boldsymbol{\Sigma}) \in \R_+^{mQ \times mQ} \\
    \boldsymbol{\Sigma} &= \text{diag}(\sigma_1^2, ..., \sigma_Q^2) \in \R_+^{Q \times Q}
\end{align*}
\noindent Thus the likelihood can be expressed as
\begin{equation}
    \boldsymbol{y} \mid \boldsymbol{X}, \boldsymbol{\Phi}, \boldsymbol{A}, \boldsymbol{B}, \{\sigma_q^2\}{q = 1}^Q \sim N_{mQ}(\boldsymbol{\boldsymbol{\mu}_{\boldsymbol{y}}, \tilde{\Sigma}})
\end{equation}
Noting the independence among the elements of $\boldsymbol{B}$ in the prior, the joint prior for $\boldsymbol{\beta}$ is
\begin{equation*}
    \boldsymbol{\beta} \mid \{\lambda_{p,q}\}_{p,q = 1}^{P,Q}, \{\tau_q\}_{q = 1}^Q \sim N_{PQ}(\boldsymbol{0}, \boldsymbol{\Lambda})
\end{equation*}

\noindent where $\boldsymbol{\Lambda} = \text{diag}(\lambda_{1,1}^2\tau_1^2, ..., \lambda_{1,Q}^2\tau_Q^2, ..., \lambda_{P,Q}^2\tau_Q^2) \in \R_+^{PQ \times PQ}$. Further, the independence on the prior for elements of $\boldsymbol{A}$ induces the following prior upon $\boldsymbol{a}$:
\begin{equation*}
    \boldsymbol{a} \sim N_{Q^2}(\boldsymbol{0}, \bI_{Q^2})
\end{equation*}



\subsection{Full conditional of $\boldsymbol{B}$($\boldsymbol{\beta}$)}

The full conditional of $\boldsymbol{\beta}$ is derived as follows;
\begin{align*}
    p(\boldsymbol{\beta} \mid \cdot) &\propto N_{mQ}(\boldsymbol{y} \mid \boldsymbol{\mu}_{\boldsymbol{y}}, \tilde{\boldsymbol{\Sigma}}) N_{PQ}(\boldsymbol{\beta} \mid \boldsymbol{0}, \boldsymbol{\Lambda}) \\
    &\propto \exp\left\{-\frac{1}{2}(\boldsymbol{y} - \boldsymbol{\mu}_{\boldsymbol{y}})^T\tilde{\boldsymbol{\Sigma}}^{-1}(\boldsymbol{y} - \boldsymbol{\mu}_{\boldsymbol{y}})\right\}\exp\left\{-\frac{1}{2}\boldsymbol{\beta}^T\boldsymbol{\Lambda}^{-1}\boldsymbol{\beta}\right\} \\
    &= \exp\left\{-\frac{1}{2}\left[(\boldsymbol{y} - (\boldsymbol{\Phi}\boldsymbol{X} \otimes \boldsymbol{A})\boldsymbol{\beta})^T\tilde{\boldsymbol{\Sigma}}^{-1}(\boldsymbol{y} - (\boldsymbol{\Phi}\boldsymbol{X} \otimes \boldsymbol{A})\boldsymbol{\beta}) + \boldsymbol{\beta}^T\boldsymbol{\Lambda}^{-1}\boldsymbol{\beta}\right]\right\} \\
    &= \exp\left\{-\frac{1}{2}\left[\boldsymbol{y}^T\tilde{\boldsymbol{\Sigma}}^{-1}\boldsymbol{y} - 2\boldsymbol{\beta}^T(\boldsymbol{\Phi}\boldsymbol{X} \otimes \boldsymbol{A})^T\tilde{\boldsymbol{\Sigma}}^{-1}\boldsymbol{y} + \boldsymbol{\beta}^T(\boldsymbol{\Phi}\boldsymbol{X} \otimes \boldsymbol{A})^T\tilde{\boldsymbol{\Sigma}}^{-1}(\boldsymbol{\Phi}\boldsymbol{X} \otimes \boldsymbol{A})\boldsymbol{\beta} + \boldsymbol{\beta}^T\boldsymbol{\Lambda}^{-1}\boldsymbol{\beta}\right]\right\} \\
    &\propto \exp\left\{-\frac{1}{2}\left[\boldsymbol{\beta}^T(\tilde{\boldsymbol{X}}^T\tilde{\boldsymbol{X}} + \boldsymbol{\Lambda}^{-1})\boldsymbol{\beta} - 2\boldsymbol{\beta}^T\tilde{\boldsymbol{X}}^T\tilde{\boldsymbol{y}}\right]\right\}
\end{align*}
\noindent Where $\tilde{\boldsymbol{y}} = \tilde{\boldsymbol{\Sigma}}^{-1/2}\boldsymbol{y}$, $\tilde{\boldsymbol{X}} = \tilde{\boldsymbol{\Sigma}}^{-1/2}(\boldsymbol{\Phi}\boldsymbol{X} \otimes \boldsymbol{A})$, $\boldsymbol{\Omega}_B = (\tilde{\boldsymbol{X}}^T\tilde{\boldsymbol{X}} + \boldsymbol{\Lambda}^{-1})$. Thus identifying the multivariate normal kernel we find,
\begin{equation}
    \boldsymbol{\beta} \mid \cdot \sim N_{PQ}(\boldsymbol{\Omega}_B^{-1}\tilde{\boldsymbol{X}}^T\tilde{\boldsymbol{y}}, \boldsymbol{\Omega}_B^{-1})
\end{equation}

\subsection{Full Conditional of $\boldsymbol{A}$($\boldsymbol{a}$)}

The full conditional of $\boldsymbol{a}$ is derived below:
\begin{align*}
    p(\boldsymbol{a} \mid \cdot) &\propto N_{mQ}(\boldsymbol{y} \mid \boldsymbol{\mu}_{\boldsymbol{y}}, \tilde{\boldsymbol{\Sigma}}) N_{Q^2}(\boldsymbol{a} \mid \boldsymbol{0}, \mathbf{I}_{Q^2}) \\
    &\propto \exp\left\{-\frac{1}{2}(\boldsymbol{y} - \boldsymbol{\mu}_{\boldsymbol{y}})^T\tilde{\boldsymbol{\Sigma}}^{-1}(\boldsymbol{y} - \boldsymbol{\mu}_{\boldsymbol{y}})\right\} \exp\left\{-\frac{1}{2}\boldsymbol{a}^T\boldsymbol{a}\right\} \\
    &= \exp\left\{-\frac{1}{2}\left[(\boldsymbol{y} - (\boldsymbol{\Phi}\boldsymbol{X}\boldsymbol{B} \otimes \mathbf{I}_{Q})\boldsymbol{a})^T\tilde{\boldsymbol{\Sigma}}^{-1}(\boldsymbol{y} - (\boldsymbol{\Phi}\boldsymbol{X}\boldsymbol{B} \otimes \mathbf{I}_{Q})\boldsymbol{a}) + \boldsymbol{a}^T\boldsymbol{a}\right]\right\} \\
    &\propto \exp\left\{-\frac{1}{2}\left[-2\boldsymbol{a}^T(\boldsymbol{\Phi}\boldsymbol{X}\boldsymbol{B} \otimes \mathbf{I}_{Q})^T\boldsymbol{\Sigma}^{-1}\boldsymbol{y} + \boldsymbol{a}^T(\boldsymbol{\Phi}\boldsymbol{X}\boldsymbol{B} \otimes \mathbf{I}_{Q})^T\boldsymbol{\Sigma}^{-1}(\boldsymbol{\Phi}\boldsymbol{X}\boldsymbol{B} \otimes \mathbf{I}_{Q})\boldsymbol{a} + \boldsymbol{a}^T\boldsymbol{a}\right]\right\} \\
    &= \exp\left\{-\frac{1}{2}\left[\boldsymbol{a}^T((\boldsymbol{\Phi}\boldsymbol{X}\boldsymbol{B} \otimes \mathbf{I}_{Q})^T\boldsymbol{\Sigma}^{-1}(\boldsymbol{\Phi}\boldsymbol{X}\boldsymbol{B} \otimes \mathbf{I}_{Q}) + \mathbf{I}_{Q^2})\boldsymbol{a} - 2\boldsymbol{a}^T(\boldsymbol{\Phi}\boldsymbol{X}\boldsymbol{B} \otimes \mathbf{I}_{Q})^T\boldsymbol{\Sigma}^{-1}\boldsymbol{y}\right]\right\} \\
    &= \exp\left\{-\frac{1}{2}\left[\boldsymbol{a}^T\boldsymbol{\Omega}_A\boldsymbol{a} - 2\boldsymbol{a}^T(\boldsymbol{\Phi}\boldsymbol{X}\boldsymbol{B} \otimes \mathbf{I}_{Q})^T\boldsymbol{\Sigma}^{-1}\boldsymbol{y}\right]\right\} \\
    &\propto \exp\left\{-\frac{1}{2}(\boldsymbol{a} - \boldsymbol{\Omega}_A^{-1}\boldsymbol{X}_*^T\tilde{\boldsymbol{y}})^T\boldsymbol{\Omega}_A(\boldsymbol{a} - \boldsymbol{\Omega}_A^{-1}\boldsymbol{X}_*^T\tilde{\boldsymbol{y}})\right\}
\end{align*}
\noindent Thus identifying the multivariate normal distribution kernel, the full conditional is,
\begin{equation}
    \boldsymbol{a} \mid \cdot \sim N_{Q^2}(\boldsymbol{\Omega}_A^{-1}\boldsymbol{X}_{*}^T\tilde{\boldsymbol{y}}, \boldsymbol{\Omega}_A^{-1})
\end{equation}
\noindent where $
    \boldsymbol{\Omega}_A = (\boldsymbol{X}_*^T\boldsymbol{X}_* + \bI_{Q^2})$, $
    \boldsymbol{X}_{*} = \tilde{\boldsymbol{\Sigma}}^{-1/2}(\boldsymbol{\Phi}\boldsymbol{X}\boldsymbol{B} \otimes \bI_{Q})$, and $
    \tilde{\boldsymbol{y}} = \tilde{\boldsymbol{\Sigma}}^{-1/2}\boldsymbol{y}$.

\subsection{Full Conditional of $\sigma_q^2$}

We derive the full conditional of $\sigma_q^2$ as follows:
\begin{align*}
    p(\sigma_q^2 \mid \cdot) &\propto N_{mQ}(\boldsymbol{y} \mid \boldsymbol{\mu}_{\boldsymbol{y}}, \tilde{\boldsymbol{\Sigma}}) \frac{1}{\sigma_q^2} \\
    &\propto \det(\tilde{\boldsymbol{\Sigma}})^{-1/2} \exp \left\{-\frac{1}{2}(\boldsymbol{y} - \boldsymbol{\mu}_{\boldsymbol{y}})^T\tilde{\boldsymbol{\Sigma}}^{-1}(\boldsymbol{y} - \boldsymbol{\mu}_{\boldsymbol{y}})\right\}\frac{1}{\sigma_q^2} \\
    &\propto (\sigma_q^2)^{-m/2 - 1} \exp\left\{-\frac{1}{2}\sum_{q' = 1}^Q\sum_{i = 1}^m(\boldsymbol{y} - \boldsymbol{\mu}_{\boldsymbol{y}})_{q'i}^2/\sigma^2_{q'}\right\} \\
    &\propto (\sigma_q^2)^{-m/2 - 1} \exp\left\{-\frac{1}{2\sigma^2_q}{\boldsymbol{y}}_q^T{\boldsymbol{y}}_q\right\}
\end{align*}
\noindent Utilizing the fact that $\tilde{\boldsymbol{\Sigma}}$ is a diagonal matrix and hence $\text{det}(\tilde{\boldsymbol{\Sigma}}) = \prod_{q = 1}^Q \prod_{i = 1}^m \sigma_q^2 = \prod_{q = 1}^Q (\sigma_q^2)^m$, and that $\tilde{\boldsymbol{\Sigma}}^{-1} = (\boldsymbol{\Sigma}^{-1}, ..., \boldsymbol{\Sigma}^{-1}) $ and $\boldsymbol{\Sigma}^{-1} = \text{diag}(1/\sigma_1^2, ..., 1/\sigma_Q^2)$. Letting,
\begin{align*}
    \boldsymbol{y}_q &= \boldsymbol{\Gamma}_q(\boldsymbol{y} - \boldsymbol{\mu}_{\boldsymbol{y}}) \in \R^{mQ \times 1} \\
    \boldsymbol{\Gamma}_q &= \operatorname{diag}(\boldsymbol{1}_q,\dots,\boldsymbol{1}_q) \in \mathbb{R}^{mQ \times mQ} \\
    \boldsymbol{1}_q &\in \R^{Q \times 1} \text{ is a vector of zeros with 1 in the $q^{\text{th}}$ position.}
\end{align*}
\noindent Thus, 

\begin{equation}
    \sigma_q^2 \mid \cdot \sim IG( \frac{m}{2}, \frac{{\boldsymbol{y}}_q^T{\boldsymbol{y}}_q}{2})
\end{equation}

\subsection{Full conditional of $\lambda_{p,q}$}

The full conditional for $\lambda_{p,q}$ is found as:
\begin{align*}
    p(\lambda_{p,q} \mid \cdot) &\propto N_{PQ}(\boldsymbol{\beta} \mid \boldsymbol{0}, \boldsymbol{\Lambda}) C^+(\lambda_{p,q} \mid 0, 1) \\
    &\propto \text{det}(\boldsymbol{\Lambda})^{-1} \exp\{-\frac{1}{2}\boldsymbol{\beta}^T\boldsymbol{\Lambda}^{-1}\boldsymbol{\beta}\} \frac{1}{1 + \lambda_{p,q}^2} \\
    &= \left[ \prod_{q' =1}^Q \prod_{p' = 1}^P \frac{1}{\lambda_{p',q'}^2\tau_{q'}^2} \right] \exp \{-\frac{1}{2} \sum_{q' = 1}^{Q}\sum_{p' = 1}^P b_{p',q'}^2/\lambda_{p',q'^2}\tau_{q'}^2\}\frac{1}{1 + \lambda_{p,q}^2} \\
    &\propto \frac{1}{\lambda_{p,q}^2}\exp\{-\frac{1}{2}\frac{b_{p,q}^2}{\lambda_{p,q}^2\tau_q^2}\}\frac{1}{1 + \lambda_{p,q}^2}
\end{align*}
\noindent This is not a known distribution. A slice sampler will be used as outlined in Section~\ref{sec:MCMC} to draw from this conditional distribution.

\subsection{Full conditional of $\tau_q^2$}

We find,
\begin{align*}
    p(\tau_q \mid \cdot) &\propto N_{PQ}(\boldsymbol{\beta} \mid \boldsymbol{0}, \boldsymbol{\Lambda}) C^+(\tau_{h} \mid 0, 1) \\
    &\propto \text{det}(\boldsymbol{\Lambda})^{-1} \exp\{-\frac{1}{2}\boldsymbol{\beta}^T\boldsymbol{\Lambda}^{-1}\boldsymbol{\beta}\} \frac{1}{1 + \tau_{h}^2} \\
    &= \left[ \prod_{q' =1}^Q \prod_{p' = 1}^P \frac{1}{\lambda_{p'q'}^2\tau_{q'}^2} \right] \exp \{-\frac{1}{2} \sum_{q' = 1}^{q}\sum_{p' = 1}^P b_{p'q'}^2/\lambda_{p'q'^2}\tau_{q'}^2\}\frac{1}{1 + \tau_{h}^2} \\
    &\propto \frac{1}{\tau_q^{2p}}  \exp\{-\frac{1}{2\tau_q^2} \sum_{p = 1}^P \frac{b_{pq}^2}{\lambda_{p,q}^2}\}\frac{1}{1 + \tau_{h}^2}
\end{align*}
\noindent This is an unknown distribution and the sampler is described in Section~\ref{sec:MCMC}.

\section{Gibbs Sampler}
\label{sec:MCMC}

We follow the 4-step Metropolis-within-Gibbs in Algorithm\ref{alg:gibbs_sampling} outlined in the article to cycle through updating the full conditional distributions. Step 2 leverages the spectral decomposition of $\boldsymbol{\Omega}_A^{-1}$ outlined below:
\begin{enumerate}
    \item Compute the spectral decomposition of $\boldsymbol{\Omega}_A^{-1} = \boldsymbol{Q}\boldsymbol{D}\boldsymbol{Q}^\top$
    \item Draw $\boldsymbol{z} \sim N_{Q^2}(0,1)$
    \item Let $\boldsymbol{a}^* = \boldsymbol{\Omega}_A^{-1}\boldsymbol{X}_{*}^T\tilde{\boldsymbol{y}} + \boldsymbol{Q}\boldsymbol{D}^{1/2}\boldsymbol{z}$
\end{enumerate}
Step 4 employs a slice sampling algorithm following  \cite{polson2014bayesian}:
\begin{enumerate}
    \item Sample $u$ uniformly on the interval $(0, \frac{1}{1 + \eta})$. 
    \item Sample $\eta$ from a truncated Exp($2/\mu$) with zero probability outside the interval $(0, \frac{1 - u}{u})$,
\end{enumerate} where $\eta = 1/\lambda_{p,q}^2 \text{ or } 1/\tau_q^2$. If $\eta = 1/\lambda_{p,q}^2$ then $\mu = \frac{b_{pq}^2}{\tau_q^2}$, otherwise for $\eta = 1/\tau_q^2 $  let $\mu = \sum_{p = 1}^P \frac{b_{pq}^2}{\lambda_{p,q}^2}$.


\section{Neuroimaging Data Application}

\small
\begin{longtable}{ccl}
\caption{53 ROIs found in the 60 significant ROI-ROI connections with corresponding shorthand codes (short hand codes following \cite{dimartino2008functional, desikan2006automated, diedrichsen2009probabilistic}).} 
\label{tab:roi_key} \\
\hline
\textbf{Index} & \textbf{Code} & \textbf{Region of Interest (ROI)} \\
\hline
\endhead

\hline
\endfoot

\hline
\endlastfoot

1  & \texttt{L\_IPL}       & Left inferior parietal lobule \\
2  & \texttt{L\_VSs}       & Left superior ventral striatum \\
3  & \texttt{L\_AMYG}      & Left amygdala \\
4  & \texttt{L\_DCaud}     & Left dorsal caudate \\
5  & \texttt{R\_VSs}       & Right superior ventral striatum \\
6  & \texttt{R\_LPC}       & Right lateral parietal cortex \\
7  & \texttt{PCC}          & Posterior cingulate cortex \\
8  & \texttt{R\_Cereb\_VI} & Right Cerebellar Lobule VI \\
9  & \texttt{L\_dlPFC}     & Left dorsolateral prefrontal cortex \\
10 & \texttt{R\_Cereb\_V}  & Right Cerebellar Lobule V \\
11 & \texttt{L\_lPMC}      & Left lateral premotor cortex \\
12 & \texttt{RSC}          & Retro Splenial \\
13 & \texttt{R\_dlPFC}     & Right dorsolateral prefrontal cortex \\
14 & \texttt{vmPFC}        & Ventromedial prefrontal cortex \\
15 & \texttt{L\_PHG}       & Left parahippocampus gyrus \\
16 & \texttt{R\_PHG}       & Right parahippocampus gyrus \\
17 & \texttt{R\_IPS}       & Right intraparietal sulcus \\
18 & \texttt{R\_SFC}       & Right superior frontal cortex \\
19 & \texttt{L\_vlPMC}     & Left ventral lateral premotor cortex \\
20 & \texttt{L\_OFC}       & Left orbitofrontal cortex \\
21 & \texttt{R\_lPFC}      & Right lateral prefrontal cortex \\
22 & \texttt{L\_PCun}      & Left precuneus \\
23 & \texttt{L\_PHC}       & Left parahippocampus \\
24 & \texttt{L\_IPS}       & Left intraparietal sulcus \\
25 & \texttt{R\_CrusI}     & Right Cerebellar Crus I \\
26 & \texttt{R\_VSi}       & Right inferior ventral striatum \\
27 & \texttt{R\_CrusII}    & Right Cerebellar Crus II \\
28 & \texttt{R\_aINS}      & Right anterior insula \\
29 & \texttt{L\_vrPUT}     & Left ventral rostral putamen \\
30 & \texttt{L\_drPUT}     & Left dorsal rostral putamen \\
31 & \texttt{L\_VSi}       & Left inferior ventral striatum \\
32 & \texttt{L\_aINS}      & Left anterior insula \\
33 & \texttt{R\_OFC}       & Right orbitofrontal cortex \\
34 & \texttt{amPFC}        & Anterior medial prefrontal cortex \\
35 & \texttt{L\_SMC}       & Left sensorimotor cortex \\
36 & \texttt{L\_FG}        & Left frontal gyrus \\
37 & \texttt{R\_DCaud}     & Right dorsal caudate \\
38 & \texttt{PCC\_2}       & Posterior cingulate cortex 2 \\
39 & \texttt{L\_aPFC}      & Left anterior prefrontal cortex \\
40 & \texttt{R\_PHC}       & Right parahippocampus \\
41 & \texttt{L\_dlPMC}     & Left dorsolateral premotor cortex \\
42 & \texttt{L\_M1}        & Left primary motor cortex \\
43 & \texttt{R\_vlPMC}     & Right ventral lateral premotor cortex \\
44 & \texttt{R\_SMC}       & Right sensorimotor cortex \\
45 & \texttt{R\_M1}        & Right primary motor cortex \\
46 & \texttt{R\_PCun}      & Right precuneus \\
47 & \texttt{L\_dcPUT}     & Left dorsal caudal putamen \\
48 & \texttt{R\_vrPUT}     & Right ventral rostral putamen \\
49 & \texttt{R\_dcPUT}     & Right dorsal caudal putamen \\
50 & \texttt{R\_drPUT}     & Right dorsal rostral putamen \\
51 & \texttt{dACC}         & Dorsal anterior cingulate cortex \\
52 & \texttt{mPFC}         & Mesial prefrontal cortex \\
53 & \texttt{R\_ITL}       & Right inferior temporal lobe \\
\end{longtable}

\section{Theoretical Guarantees: Proofs}
\underline{\textbf{Proof of Lemma~\ref{lem:prior_conc_main}}}\\
\textbf{Step 1:} Note that $\text{Var}(\text{vec}(\bY_{\bPhi}))$ under the fitted model has is $\bI_{mQ}$ while $\text{Var}(\text{vec}(\bY_{\bPhi}))$ under the true data generating sketched model is $\bI_Q \otimes \bPhi\bPhi^T$, we compute the KL divergence between the two fitted densities. Using the formula for KL divergence between two Gaussian matrix-variate densities:
\begin{align*}
\mathrm{KL}\left(p^{(m)}(\bY_{\bPhi}\mid\bX_{\bPhi},\bC^*),\, p^{(m)}(\bY_{\bPhi}\mid\bX_{\bPhi},\bC)\right)
= \frac{Q}{2}\left[\mathrm{tr}(\bPhi\bPhi^T - \bI_m) - \log\det(\bPhi\bPhi^T)\right] + \frac{1}{2}\|\bX_{\bPhi}(\bC - \bC^*)\|_F^2,
\end{align*}
where we used $\det(\bI_Q \otimes \bPhi\bPhi^T) = (\det(\bPhi\bPhi^T))^Q$ and $\mathrm{tr}(\bI_Q \otimes \bPhi\bPhi^T) = Q \cdot \mathrm{tr}(\bPhi\bPhi^T)$. Since $\mathrm{tr}(\bH) - \log\det(\bH) - d \leq \|\bH - \bI_d\|_F^2$ for any $d\times d$ positive definite matrix $\bH$, we obtain:
\begin{align}
\mathrm{KL}\left(p^{(m)}(\bY_{\bPhi}\mid\bX_{\bPhi},\bC^*),\, p^{(m)}(\bY_{\bPhi}\mid\bX_{\bPhi},\bC)\right)
\leq \frac{Q}{2}\|\bPhi\bPhi^T - \bI_m\|_F^2 + \frac{1}{2}\|\bX_{\bPhi}(\bC - \bC^*)\|_F^2. \tag{S1}
\end{align}

\noindent\textbf{Step 2:} By Assumptions 1 and 5:
\[
\frac{Q}{2}\|\bPhi\bPhi^T - \bI_m\|_F^2 \leq \frac{Q c''^2 m}{2n} \leq \frac{m\epsilon_m^2}{2},
\]
where the last inequality holds since $m/n=o(m\epsilon_m^2)$ under Assumption 4.

\noindent\textbf{Step 3:} For the second term in (S1), using the Cauchy–Schwarz inequality:
\[
\|\bX_{\bPhi}(\bC - \bC^*)\|_F^2 \leq \|\bX_{\bPhi}\|_F^2 \|\bC - \bC^*\|_F^2 \leq Q \max_{1\leq j\leq P}\|\bX_{\bPhi j}\|^2 \|\bC - \bC^*\|_F^2.
\]
By Assumption 5, $\lambda_{\max}(\bPhi\bPhi^T) \leq 1 + c''\sqrt{m/n}$, so:
\[
\|\bX_{\bPhi j}\|^2 = \|\bPhi \bX_j\|^2 \leq \lambda_{\max}(\bPhi\bPhi^T) \|\bX_j\|^2 \leq \left(1 + c''\sqrt{m/n}\right)\|\bX_j\|^2 = O(n),
\]
where the last step uses Assumption 3. Therefore:
\begin{equation}
\|\bX_{\bPhi}(\bC - \bC^*)\|_F^2 \leq Q \cdot O(n) \cdot \|\bC - \bC^*\|_F^2. \tag{S2}
\end{equation}

It follows that:
\[
\pi\left(\bC : \|\bX_{\bPhi}(\bC - \bC^*)\|_F^2 \leq m\epsilon_m^2\right) \geq 
\pi\left(\bC : \|\bC - \bC^*\|_F^2 \leq \frac{m\epsilon_m^2}{Q \cdot O(n)}\right) 
= \pi\left(\bC : \|\bC - \bC^*\|_F^2 \leq \delta_C^2\right),
\]
where $\delta_C = \{(QR^* + R^* s^* \log P)/n\}^{1/2}$, since $m\epsilon_m^2/(Qn) = (QR^* + R^* s^* \log P)/n = \delta_C^2$ by definition of $\epsilon_m^2$. Assumption 1 has also been used here.

\noindent\textbf{Step 4:} Under Assumption 2, By Lemma S5 of the supplementary material of \cite{chakraborty2020bayesian}, the prior $\pi(\bC)$ satisfies:
\[
\pi\left(\bC : \|\bC - \bC^*\|_F < \delta_\bC\right) \geq e^{-K(QR^* + R^*s^*\log P)} = e^{-K m\epsilon_m^2},
\]
for some positive constant $K$.

From Steps 2–4, the KL neighborhood $B_m^*(\bC^*, \epsilon_m)$ contains the set $\{\bC : \|\bX_{\bPhi}(\bC - \bC^*)\|_F^2 \leq m\epsilon_m^2\}$, which in turn contains a Frobenius ball of radius $\delta_\bC$ around $\bC^*$. Therefore:
\[
\pi\left(B_m^*(\bC^*, \epsilon_m)\right) \geq \pi\left(\bC : \|\bC - \bC^*\|_F^2 \leq \delta_\bC^2\right) \geq e^{-K m\epsilon_m^2}. \qquad \blacksquare
\]

\noindent\underline{\textbf{Proof of Lemma~\ref{thm:Dm_lower_main}}}\\
The proof follows three steps.

\noindent\textbf{Step 1:} Since $D_m \geq 0$, restricting the integral to $B_m^*(\bC^*, \epsilon_m)$ gives:
\[
D_m
\;\geq\;
\pi(B_m^*(\bC^*, \epsilon_m))
\int_{B_m^*(\bC^*, \epsilon_m)} e^{-\alpha r(\bC,\bC^*)}\,d\pi_\bC^{(B)}
\;=:\;
\pi(B_m^*(\bC^*, \epsilon_m)) \cdot D_m^*,
\]
where $d\pi_\bC^{(B)} = \pi(B_m^*(\bC^*, \epsilon_m))^{-1}d\pi_\bC$ is the prior restricted to
$B_m^*(\bC^*, \epsilon_m)$. By Lemma~5.1:
\begin{equation}
\pi(B_m^*(\bC^*, \epsilon_m)) \;\geq\; e^{-Km\epsilon_m^2}. \tag{S3}
\end{equation}

\noindent\textbf{Step 2:} Applying Jensen's inequality to the concave function $\log(\cdot)$:
\[
\log D_m^*
\;\geq\;
-\alpha \int_{B_m^*} r(\bC, \bC^*)\,d\pi_C^{(B)}
\;=\; Z.
\]
Taking expectation under $P_{\bC^*}^{(m)}$:
\begin{align}
E_{\bC^*}[Z]
&= -\alpha \int_{B_m^*(\bC^*, \epsilon_m)} E_{\bC^*}[r(\bC,\bC^*)]\,d\pi_\bC^{(B)}
\nonumber\\
&= -\alpha \int_{B_m^*(\bC^*, \epsilon_m)}
\mathrm{KL}\!\left(p^{(m)}(\bY_{\bPhi}|\bX_{\bPhi},\bC^*), p^{(m)}(\bY_{\bPhi}|\bX_{\bPhi},\bC)\right)d\pi_\bC^{(B)}
\;\geq\; -\alpha m\epsilon_m^2,
\tag{S4}
\end{align}
where the last inequality follows because, by definition of
$B_m^*(\bC^*, \epsilon_m)$, each integrand is at most $m\epsilon_m^2$.

\noindent\textbf{Step 3:}
We now compute $\mathrm{Var}_{\bC^*}(Z)$.
Since $Z = -\alpha\int_{B_m^*(\bC^*, \epsilon_m)} r(\bC,\bC^*)\,d\pi_\bC^{(B)}$:
\[
\mathrm{Var}_{\bC^*}(Z)
\;\leq\;
\alpha^2
\int_{B_m^*(\bC^*, \epsilon_m)}
\mathrm{Var}_{\bC^*}\!\left(r(\bC,\bC^*)\right)d\pi_\bC^{(B)},
\]
where the variance of the log-likelihood ratio is:
\begin{align*}
\mathrm{Var}_{\bC^*}(r(\bC,\bC^*))
&=
\int p^{(m)}(\bY_{\bPhi}|\bX_{\bPhi},\bC^*)
\left(
\log\frac{p^{(m)}(\bY_{\bPhi}|\bX_{\bPhi},\bC^*)}
         {p^{(m)}(\bY_{\bPhi}|\bX_{\bPhi},\bC)}
\right)^2
d\bY_{\bPhi}\\
&-
\left[\mathrm{KL}\left(p^{(m)}(\bY_{\bPhi}|\bX_{\bPhi},\bC^*), p^{(m)}(\bY_{\bPhi}|\bX_{\bPhi},\bC)\right)\right]^2.
\end{align*}
From the second-moment calculation,
with $\bDelta = \bX_{\bPhi}(\bC - \bC^*)$:
\[
\mathrm{Var}_{\bC^*}(r(\bC,\bC^*)) =\underbrace{
\frac{Q}{2}\|\bPhi\bPhi^T - \bI_m\|_F^2
}_{T_1}
+
\underbrace{
\mathrm{tr}[\bDelta^T \bPhi\bPhi^T \bDelta]
}_{T_2}.
\]
We bound each term on $B_m^*(\bC^*, \epsilon_m)$.

\begin{itemize}

\item \textbf{Bound on $T_1$:}
By Assumptions 1, 4, 5, $\|\bPhi\bPhi^T - \bI_m\|_F^2 \leq c''^2 m/n$, so:
\[
T_1
= \frac{Q}{2}\|\bPhi\bPhi^T - \bI_m\|_F^2
\leq \frac{Qc''^2 m}{2n}
\leq K_1 m\epsilon_m^2.
\]

\item \textbf{Bound on $T_2$:}
By Assumption~5, $\lambda_{\max}(\bPhi\bPhi^T) \leq 1 + c''\sqrt{m/n}$, so:
\[
T_2
= \mathrm{tr}[\bDelta^T \bPhi\bPhi^T \bDelta]
\leq \lambda_{\max}(\bPhi\bPhi^T)\|\bDelta\|_F^2
\leq
\left(1 + c''\sqrt{m/n}\right)
\|\bX_{\bPhi}(\bC-\bC^*)\|_F^2
\leq K_2 m\epsilon_m^2,
\]
where the last inequality uses $\|\bX_{\bPhi}(\bC-\bC^*)\|_F^2 \leq m\epsilon_m^2$
on $B_m^*(\bC^*,\epsilon_m)$, which follows from the KL bound established in Step~1 of Lemma~5.1.

\end{itemize}

\noindent
Therefore, with $K' = K_1 + K_2$:
\begin{equation}
\mathrm{Var}_{\bC^*}(Z)
\;\leq\;
\alpha^2
\int_{B_m^*(\bC^*,\epsilon_m)}(T_1 + T_2)\,d\pi_\bC^{(B)}
\;\leq\;
K'\alpha^2 m\epsilon_m^2.
\tag{S5}
\end{equation}

\noindent
Applying \textbf{Chebyshev's inequality} and equations (S4) and (S5):
\[
P_{\bC^*}^{(m)}\!\left(Z \leq -\alpha(D+t)m\epsilon_m^2\right)
\;\leq\;
\frac{\mathrm{Var}_{\bC^*}(Z)}
     {\{\alpha(D-1+t)m\epsilon_m^2\}^2}
\;\leq\;
\frac{K'}{(D-1+t)^2 m\epsilon_m^2}.
\]
Hence, with $P_{\bC^*}^{(m)}$-probability at least
$1 - K'/\{(D-1+t)^2 m\epsilon_m^2\}$:
\[
\log D_m^* \;\geq\; Z \;\geq\; -\alpha(D+t)m\epsilon_m^2
\;\implies\;
D_m^* \;\geq\; e^{-\alpha(D+t)m\epsilon_m^2}.
\]

Using equation (S3) and the bound on $D_m^*$, and noting that
$e^{-Km\epsilon_m^2}$ gets absorbed into $e^{-(D+t)m\epsilon_m^2}$
for $D > 1$:
\[
D_m
\;\geq\;
\pi(B_m^*(\bC^*,\epsilon_m))\cdot D_m^*
\;\geq\;
e^{-(D+t)m\epsilon_m^2}.
\qquad \blacksquare
\]

\noindent\underline{\textbf{Proof of Theorem~\ref{thm:renyi_main}}}\\
Write the fractional posterior probability as a ratio:
\[
\Pi_\alpha(U_m \mid \bY_{\bPhi})
\;=\;
\frac{N_m}{D_m},
\quad
N_m
=
\int_{U_m} e^{-\alpha r(\bC,\bC^*)}\pi(\bC)\,d\bC,\:D_m
=
\int e^{-\alpha r(\bC,\bC^*)}\pi(\bC)\,d\bC
\]

\noindent
Using the identity
\[
A_\alpha\!\left(
p^{(m)}(\bY_{\bPhi}|\bX_{\bPhi},\bC^*),\,
p^{(m)}(\bY_{\bPhi}|\bX_{\bPhi},\bC)
\right)
=
e^{-(1-\alpha)D_\alpha(p^{(m)}(\bY_{\bPhi}|\bX_{\bPhi},\bC^*),\,
p^{(m)}(\bY_{\bPhi}|\bX_{\bPhi},\bC))}
\]
and Fubini's theorem:
\[
E_{\bC^*}[N_m]
=
\int_{U_m}
E_{\bC^*}\!\left[e^{-\alpha r(\bC,\bC^*)}\right]
\pi(\bC)\,d\bC
=
\int_{U_m}
e^{-(1-\alpha)D_\alpha(p^{(m)}(\bY_{\bPhi}|\bX_{\bPhi},\bC^*),\,
p^{(m)}(\bY_{\bPhi}|\bX_{\bPhi},\bC))}
\pi(\bC)\,d\bC.
\]
On $U_m$, by definition,
$(1-\alpha)D_\alpha(p^{(m)}(\bY_{\bPhi}|\bX_{\bPhi},\bC^*),\,
p^{(m)}(\bY_{\bPhi}|\bX_{\bPhi},\bC)) \geq (D+3t)m\epsilon_m^2$,
so:
\[
E_{\bC^*}[N_m]
\;\leq\;
e^{-(D+3t)m\epsilon_m^2}
\cdot
\int_{U_m}\pi(\bC)\,d\bC
\;\leq\;
e^{-(D+3t)m\epsilon_m^2}.
\]
Applying \textbf{Markov's inequality}:
\begin{equation}
P_{\bC^*}^{(m)}\!\left(
N_m \geq e^{-(D+2t)m\epsilon_m^2}
\right)
\;\leq\;
\frac{E_{\bC^*}[N_m]}{e^{-(D+2t)m\epsilon_m^2}}
\;\leq\;
e^{-tm\epsilon_m^2}
\;\leq\;
\frac{1}{(D-1+t)^2\,m\epsilon_m^2}.
\tag{S6}
\end{equation}
Hence, with $P_{\bC^*}^{(m)}$-probability at least
$1 - 1/\{(D-1+t)^2 m\epsilon_m^2\}$:
\begin{equation}
N_m \;\leq\; e^{-(D+2t)m\epsilon_m^2}.
\tag{S7}
\end{equation}
From Theorem~\ref{thm:Dm_lower_main}, $D_m\geq e^{-(D+t)m\epsilon_m^2}$ with $P_{\bC^*}^{(m)}$-probability at least
$1 - K'/\{(D-1+t)^2 m\epsilon_m^2\}$, for some positive constant $K'$. Combining the aforementioned two results, we obtain 
\begin{align*}
\Pi_\alpha(U_m \mid \bY_{\bPhi})
\;=\;
\frac{N_m}{D_m},
\quad
N_m= \frac{e^{-(D+2t)m\epsilon_m^2}}{e^{-(D+t)m\epsilon_m^2}}\leq e^{-tm\epsilon_m^2}.
\end{align*}
with $P_{\bC^*}^{(m)}$-probability at least
$1 - K'/\{(D-1+t)^2 m\epsilon_m^2\} \qquad \blacksquare$

\noindent\underline{\textbf{Proof of Theorem~\ref{TV_thm}}}\\
Define, 
\begin{align*}
m_{\alpha}(\bY_{\bPhi})=\int \{p^{(m)}(\bY_{\bPhi}|\bX_{\bPhi},\bC)\}^\alpha\pi(\bC)d\bC,\:\: m(\bY_{\bPhi})=\int p^{(m)}(\bY_{\bPhi}|\bX_{\bPhi},\bC)\pi(\bC)d\bC.   
\end{align*}
Define, for any $M>0$,
$\mathcal{A}_M=\{\bC:p^{(m)}(\bY_{\bPhi}|\bX_{\bPhi},\bC)\leq M\}$. On 
$\mathcal{A}_M$, $\{p^{(m)}(\bY_{\bPhi}|\bX_{\bPhi},\bC)\}^\alpha\leq \max\{1,M^\alpha\}\leq \max\{1,M\}$. As $\alpha\rightarrow 1$, $\{p^{(m)}(\bY_{\bPhi}|\bX_{\bPhi},\bC)\}^\alpha\pi(\bC)\rightarrow p^{(m)}(\bY_{\bPhi}|\bX_{\bPhi},\bC)\pi(\bC)$, and $\int_{\mathcal{A}_M}p^{(m)}(\bY_{\bPhi}|\bX_{\bPhi},\bC)\pi(\bC)d\bC\leq \infty$, by dominated convergence theorem,
\begin{align*}
\int_{\mathcal{A}_M}\{p^{(m)}(\bY_{\bPhi}|\bX_{\bPhi},\bC)\}^\alpha\pi(\bC)d\bC\rightarrow \int_{\mathcal{A}_M}p^{(m)}(\bY_{\bPhi}|\bX_{\bPhi},\bC)\pi(\bC)d\bC,\:\:\mbox{as}\:\alpha\rightarrow 1.   
\end{align*}
 $\pi(\mathcal{A}_M^c)=\pi(\bC:p^{(m)}(\bY_{\bPhi}|\bX_{\bPhi},\bC)\leq M)\leq \frac{m(\bY_{\bPhi})}{M}$, by Markov inequality.

 Using $x^\alpha<x$ for $x>1$, we have
 \begin{align*}
 \int_{\mathcal{A}_M^c} \{p^{(m)}(\bY_{\bPhi}|\bX_{\bPhi},\bC)\}^\alpha\pi(\bC)d\bC\leq  \int_{\mathcal{A}_M^c} p^{(m)}(\bY_{\bPhi}|\bX_{\bPhi},\bC)\pi(\bC)=\int_{\bC:p^{(m)}>M} p^{(m)}(\bY_{\bPhi}|\bX_{\bPhi},\bC)\pi(\bC) \rightarrow 0,
 \end{align*}
 as $M\rightarrow\infty$.

 For any $\epsilon > 0$, choose $M$ large enough so that
\[
\int_{\mathcal{A}_M^c} p^{(m)}(\bY_{\bPhi}|\bX_{\bPhi},\bC) \pi(\bC)d\bC < \epsilon/3.
\]
Then for $\alpha$ sufficiently close to $1$,
\begin{align*}
|m_\alpha(\bY_{\bPhi}) - m(\bY_{\bPhi})| &\leq \left|\int_{\mathcal{A}_M} \{p^{(m)}(\bY_{\bPhi}|\bX_{\bPhi},\bC)\}^\alpha \pi(\bC) d\bC - \int_{\mathcal{A}_M} p^{(m)}(\bY_{\bPhi}|\bX_{\bPhi},\bC) \pi(\bC)d\bC\right| \\
&+ 2\int_{\mathcal{A}_M^c} p^{(m)}(\bY_{\bPhi}|\bX_{\bPhi},\bC)\pi(\bC) d\bC < \epsilon.
\end{align*}
Therefore $m_\alpha(\bY_{\bPhi}) \to m(\bY_{\bPhi})$ as $\alpha \to 1$, for $P_{\bC^*}^{(m)}$-almost every $\bY_{\bPhi}$.

For each $\bC$ and $P_{\bC^*}^{(m)}$-almost every $\bY_{\bPhi}$,
\[
\Pi_{\alpha}(\bC \mid \bY_{\bPhi}) = \frac{\{p^{(m)}(\bY_{\bPhi}|\bX_{\bPhi},\bC)\}^\alpha \pi(\bC)}{m_\alpha(\bY_{\bPhi})} \to \frac{p^{(m)}(\bY_{\bPhi}|\bX_{\bPhi},\bC) \pi(\bC)}{m(\bY_{\bPhi})} = \Pi(\bC \mid \bY_{\bPhi}),
\]
pointwise as $\alpha \to 1$, since both the numerator converges pointwise and $m_\alpha(\bY_{\bPhi})\to m(\bY_{\bPhi}) > 0$.

Since $\Pi_{\alpha}(\cdot \mid \bY)$ and $\Pi(\cdot \mid \bY_{\bPhi})$ are both probability densities integrating to $1$, and the pointwise limit holds, an application of \textbf{Scheffé's theorem} gives:
\[
\mathrm{TV}(\Pi_{\alpha}(\cdot \mid \bY_{\bPhi}),\Pi(\cdot \mid \bY_{\bPhi})) = \int |\Pi_{\alpha} - \Pi|\, d\Pi \to 0, \quad \text{as } \alpha \to 1. \qquad \blacksquare
\]

\vskip 0.2in
\bibliography{references}

\end{document}